\documentclass[twoside,11pt]{article}

\usepackage{jmlr2e}
\usepackage{epsfig}
\usepackage{graphicx}
\usepackage{algorithm}
\usepackage{algorithmic,amssymb}
\usepackage{amsmath}\allowdisplaybreaks
\usepackage{natbib}
\usepackage{color}

\begin{document}

\title{Asynchronous Stochastic  Block Coordinate Descent \\ with  Variance Reduction  }

\author{\name Bin Gu \email jsgubin@gmail.com      \\
   \name Zhouyuan Huo \email {zhouyuan.huo@mavs.uta.edu} \\
   \name Heng Huang \email {heng@uta.edu} \\
\addr Department of Computer Science and Engineering\\
       University of Texas at Arlington\\
      }

\editor{}

\maketitle

\begin{abstract}%   <- trailing '%' for backward compatibility of .sty file
 Asynchronous parallel  implementations  for stochastic optimization    have received huge successes in  theory and practice recently. Asynchronous implementations with lock-free are more efficient than the one with writing or reading lock. In this paper, we focus on  a composite objective function consisting    of a  smooth convex function $f$  and a block separable convex function, which widely exists in  machine learning and computer vision. We propose an asynchronous stochastic  block coordinate descent algorithm with the accelerated technology of variance reduction  (AsySBCDVR), which are with lock-free in the implementation and analysis. AsySBCDVR is particularly  important because it  can scale well with the sample size and dimension simultaneously.  We prove that AsySBCDVR achieves a linear convergence rate when the function $f$ is with the optimal strong
convexity property, and a sublinear rate  when $f$ is with the  general convexity. More importantly, a near-linear speedup on a parallel system with shared memory can be obtained.
\end{abstract}

\begin{keywords}
stochastic optimization, block coordinate descent, parallel computing,  lock-free
\end{keywords}

\section{Introduction}
Stochastic optimization technologies in  theory and practice are emerging  recently due to the demand of handling large scale data.  Specifically, stochastic gradient descent (SGD) algorithms with various kinds of acceleration technologies \citep{zhang2004solving,johnson2013accelerating,schmidt2013minimizing,ghadimi2013stochastic,reddi2015variance}  were proposed to processing large scale smooth convex or nonconvex problems. Also, stochastic  coordinate descent (SCD) algorithms  \citep{takac2014randomized,zhao2014accelerated,lu2015complexity}  and stochastic dual coordinate ascent algorithms  \citep{shalev2014accelerated,Shalev-Shwartz16}  were proposed. For non-smoothing problems, the corresponding proximal algorithms were proposed \citep{xiao2014proximal,lin2014accelerated,nitanda2014stochastic}.
Basically, these algorithms  are sequential algorithms which can not be directly used in parallel environment.

 To scale up the stochastic optimization  algorithms,  asynchronous parallel  implementations  \citep{richtarik2016parallel,chaturapruek2015asynchronous,liu2015asynchronous2,recht2011hogwild,reddi2015variance,lian2016comprehensive,liu2015asynchronous,zhao2016fast,huo2016asynchronous,lian2015asynchronous,avron2015revisiting,hsieh2015passcode,mania2015perturbed,huo2016decoupled,huo2016distributed}   have been proposed recently, and  received huge successes. Among these asynchronous parallel implementations, the ones with lock-free are
more efficient than the ones with writing or reading lock, because they can achieve  near-linear speedup which is the ultimate goal of the parallel computation. In these paper, we focus on the asynchronous parallel implementations with lock-free.

 There have been several asynchronous parallel implementations for the stochastic optimization  algorithms which are totally free of reading and writing locks. For example, \cite{zhao2016fast} proposed an asynchronous parallel algorithm for SVRG and proved the linear convergence.  \cite{liu2015asynchronous} proposed an  asynchronous parallel algorithm for SCD and proved the linear convergence. \cite{avron2015revisiting} proposed an  asynchronous parallel algorithm for SCD to solve a linear system.   \cite{mania2015perturbed} proposed a perturbed iterate framework to analyze the asynchronous parallel algorithms of SGD, SCD and sparse SVRG. \cite{huo2016asynchronous} proposed an asynchronous  SGD algorithm with variance reduction on   non-convex optimization problems and proved the convergence.  \cite{lian2016comprehensive} proposed an asynchronous  stochastic optimization algorithm with  zeroth order and proved the convergence.

 In this paper, we focus on a composite objective function as follows.
   \begin{eqnarray}\label{formulation1}
\min_{x\in \mathbb{R}^n} F(x) = f(x)+g(x)
\end{eqnarray}
where $f(x)=\frac{1}{l}\sum_{i=1}^{l} f_i(x)$, $f_i:\mathbb{R}^n \mapsto \mathbb{R}$ is a smooth convex function. $g:\mathbb{R}^n \mapsto \mathbb{R} \cup \{ \infty \}$ is a block separable, closed, convex, and extended real-valued function. Given a partition $\{ \mathcal{G}_{1},\cdots,{\mathcal{G}_{k}} \}$ of n
coordinates of $x$,  we can write $g(x)$ as  $g(x)=\sum_{j=1}^{k} g_{\mathcal{G}_{j}}(x_{\mathcal{G}_{j}})$. The formulation (\ref{formulation1}) covers  many machine learning and computer vision
problems, for example, Lasso \citep{tibshirani1996regression}, group Lasso \citep{roth2008group}, sparse multiclass classification \citep{blondel2013block}, and so on. Note that, \cite{liu2015asynchronous}  consider the formulation (\ref{formulation1}) with the constraints that $|{\mathcal{G}_{j}}|=1$ for all $j$. Thus, the formulation in  \citep{liu2015asynchronous} is a special case of (\ref{formulation1}). Each iteration in \citep{liu2015asynchronous} only modifies a single component
of $x$  which is an atomic operation in the  parallel system
with shared memory. However, we need to modify a block coordinate ${\mathcal{G}_{j}}$ of $x$ with lock-free for each iteration, which are more complicated in the asynchronous parallel analysis  than the atomic updating in \citep{liu2015asynchronous}.  Due to the complication induced by the block representation of $g(x)$, there have been no asynchronous stochastic  block coordinate descent algorithm with lock-free proposed for handle formulation (\ref{formulation1}), especially on the theoretical analysis.

In this paper, we propose an asynchronous stochastic  block coordinate descent algorithm with the accelerated technology of variance reduction  (AsySBCDVR), which are with lock-free in the implementation and analysis. We prove that AsySBCDVR achieves a linear convergence rate when the function $f$ is with the optimal strong
convexity property, and a sublinear rate  when $f$ is with the  general convexity. More importantly, a near-linear speedup on a parallel system with shared memory can be obtained.

AsySBCDVR is particularly  important because it  can scale well with the sample size and dimension simultaneously. In the big data era, both of the  sample size and dimension  could be  huge at the same time as modern data collection technologies evolve, which demands that the learning algorithms can process large scale datasets with large sample size and high dimension.   AsySBCDVR is based on a doubly stochastic scheme which randomly choose a set of samples and a block coordinate for each iteration. Thus, AsySBCDVR can process large scale datasets.  Especially,  the technology of variance reduction is used to accelerate AsySBCDVR such that AsySBCDVR has the linear or sublinear convergence in different conditions. Thus,  AsySBCDVR can scale well with the sample size and dimension simultaneously.

We organize the rest of the paper as follows.  In Section
 \ref{section_preliminary}, we give some preliminaries. In section \ref{section_algorithm}, we propose our AsySBCDVR algorithm. In Section \ref{convergence_analysis}, we prove the  convergence rate for AsySBCDVR.  Finally, we give some concluding
remarks in Section \ref{conclusion}.
%------------------------------------------------------------------------
\section{Preliminaries}\label{section_preliminary}
In this section, we introduce  the condition of optimal strong convexity and three different Lipschitz constants and give the corresponding assumptions, which are  critical to  the  analysis of AsySBCDVR.

\noindent \textbf{Optimal Strong Convexity:} \ \
Let $F^*$ denote the optimal value of (\ref{formulation1}), and let $S$ denote the solution set of $F$ such that $F(x)=F^*$, $\forall x \in S$. Firstly, we assume that $S$ is nonempty (i.e., Assumption \ref{ass1}), which is reasonable to (\ref{formulation1}).
\begin{assumption}\label{ass1}
 The solution set $S$ of (\ref{formulation1}) is nonempty.
 \end{assumption}
 Based on $S$, we define $\mathcal{P}_{S(x)}=\arg \min_{y\in S }\| y-x \|^2$ as the Euclidean-norm projection of a vector $x$ onto $S$. Then, we assume that the convex function $f$ is with the  optimal strong convexity (i.e., Assumption \ref{assumption2}).
\begin{assumption}[Optimal strong convexity]\label{assumption2} The convex function $f$ has the condition of    optimal strong convexity with parameter $l>0$ with respect to the optimal set $S$, which means that, $\exists l$   such that, $\forall x$, we have
   \begin{eqnarray}
F(x)-F(\mathcal{P}_{S}(x)) \geq \frac{l}{2} \|x - \mathcal{P}_{S}(x) \|^2
\end{eqnarray}
\end{assumption}
As mentioned in \cite{liu2015asynchronous}, the condition of optimal strong convexity  is significantly weaker than the normal strong convexity condition. And several examples of optimally strongly convex
functions that are not strongly convex are provided in \citep{liu2015asynchronous}.

\noindent \textbf{Lipschitz Smoothness:} \ \ Let $\Delta_j$ denote the zero vector in $\mathbb{R}^n$ except that the block coordinates indexed by the set  ${\mathcal{G}_{j}}$. We define the normal Lipschitz constant ($L_{nor}$), block restricted Lipschitz constant ($L_{res} $) and block coordinate Lipschitz constant ($L_{\max}$)  as follows.
\begin{definition}[Normal Lipschitz constant]\label{definition1}
$L_{nor} $ is the normal Lipschitz constant for $\nabla f_i$ ($\forall i \in \{1,\cdots,l\}$) in (\ref{formulation1}), such that, $\forall x$ and $\forall y$, we have
   \begin{eqnarray}\label{equdef1}
\| \nabla f_i(x) - \nabla f_i(y) \|  \leq L_{nor} \|x - y \|
\end{eqnarray}
\end{definition}

\begin{definition}[Block Restricted Lipschitz constant]\label{definition2}
$L_{res} $ is the block restricted Lipschitz constant for $\nabla f_i$ ($\forall i \in \{1,\cdots,l\}$) in (\ref{formulation1}), such that, $\forall x$, and $\forall j\in \{ 1,\cdots,k\}$,  we have
   \begin{eqnarray}\label{equdef2}
\| \nabla f_i(x+ \Delta_j) - \nabla f_i(x) \|  \leq L_{res} \left \| (\Delta_j)_{\mathcal{G}_{j}} \right \|
\end{eqnarray}
\end{definition}

\begin{definition}[Block Coordinate Lipschitz constant]\label{definition2}
$L_{\max}$  is the block coordinate Lipschitz constant for $\nabla f_i$ ($\forall i \in \{1,\cdots,l\}$) in (\ref{formulation1}), such that, $\forall x$, and $\forall j\in \{ 1,\cdots,k\}$,  we have
   \begin{eqnarray}\label{coordinate_lipschitz_constant1}
\max_{j=1,\cdots,k} \| \nabla f_i(x+ \Delta_j) - \nabla f_i(x) \| \leq L_{\max} \left \| (\Delta_j)_{\mathcal{G}_{j}} \right \|
\end{eqnarray}
(\ref{coordinate_lipschitz_constant1}) is equvilent to the formulation (\ref{coordinate_lipschitz_constant2}).
   \begin{eqnarray} \label{coordinate_lipschitz_constant2}
f_i(x+ \Delta_j) \leq f_i(x) + \langle \nabla_{\mathcal{G}_{j}} f_i(x) , ( \Delta_j )_{\mathcal{G}_{j}}\rangle + \frac{L_{\max}}{2}  \left \| (\Delta_j)_{\mathcal{G}_{j}} \right \|^2
\end{eqnarray}
\end{definition}
Based on $L_{nor}$  $L_{res}$ and $L_{max}$ as defined above, we assume that the function $f_i$ ($\forall i \in \{1,\cdots,l\}$ is Lipschitz smooth with $L_{nor}$  $L_{res}$ and $L_{max}$ (i.e., Assumption \ref{assumption3}). In addition, we define $\Lambda_{res}=\frac{L_{res}}{L_{\max}}$, $\Lambda_{nor}=\frac{L_{nor}}{L_{\max}}$.
\begin{assumption}[Lipschitz smoothness]\label{assumption3}
The function $f_i$ ($\forall i \in \{1,\cdots,l\}$ is Lipschitz smooth with the  normal Lipschitz constant $L_{nor}$, block  restricted Lipschitz constant $L_{res}$ and block coordinate Lipschitz constant $L_{\max}$.
\end{assumption}

 \section{Algorithm}\label{section_algorithm}
 In this section, we propose our AsySBCDVR. AsySBCDVR is designed for  the parallel environment with shared memory, such as  multi-core processors and GPU-accelerators, but it can also work in the parallel environment with distributed memory.

In the  parallel environment with shared memory, all cores in CPU or GPU can read and write  the vector $x$ in the shared memory simultaneously without any lock.  Besides randomly choosing a sample set and a block coordinate, AsySBCDVR is also  accelerated by the variance reduction. Thus, AsySBCDVR has two-layer loops. The outer layer is to parallelly compute the full gradient
 $\nabla f(x^s) = \frac{1}{l} \sum_{i=1}^l \nabla f_i(x^s) $, where the superscript $s$ denotes the $s$-th outer loop. The inner layer is to parallelly and repeatedly  update the vector $x$ in the shared memory.
 Specifically,  all cores repeat the
following steps independently and concurrently without any lock:
 \begin{enumerate}
 \item \textbf{Read:} Read the vector $x$ from the shared memory to the local memory
without reading lock. We use $\widehat{x}^{s+1}_t$ to denote its value, where the subscript $t$ denotes the $t$-th inner loop.
 \item \textbf{Compute:} Randomly  choose a mini-batch $\mathcal{B}$ and a block  coordinate $j$ from $\{1, ...,k\}$, and locally compute  $ \widehat{v}^{s+1}_{\mathcal{G}_j} = \frac{1}{|\mathcal{B}|}\sum_{i\in \mathcal{B}} \nabla_{\mathcal{G}_j} f_i(\widehat{x}^{s+1}_{t})-\frac{1}{|\mathcal{B}|}\sum_{i\in \mathcal{B}}\nabla_{\mathcal{G}_j} f_i(\widetilde{x}^s)  + \nabla_{\mathcal{G}_j} f(\widetilde{x}^{s}) $.
 \item \textbf{Update:} Update the block $j$ of the vector $x$ in the shared memory as $(x_{t+1}^{s+1})_{\mathcal{G}_j} \leftarrow \mathcal{P}_{{\mathcal{G}_j},\frac{\gamma}{L_{\max}}g_j}\left (  (x_{t}^{s+1})_{\mathcal{G}_j} - \frac{\gamma}{L_{\max}}  \widehat{v}^{s+1}_{\mathcal{G}_j}   \right )$   without writing lock.
 \end{enumerate}
The detailed description of AsySBCDVR is  presented in  Algorithm \ref{algorithm2}.  Note that $ \widehat{v}^{s+1}_{\mathcal{G}_j}$ computed  locally is the approximation of $\nabla_{\mathcal{G}_j} f(\widehat{x}^{s+1}_{t})$, and the expectation  of $\widehat{v}^{s+1}_t$ on $\mathcal{B}$ is equal to $\nabla f(\widehat{x}^{s+1}_{t})$ as follows.
\begin{eqnarray}
\mathbb{E}\widehat{v}^{s+1}_{t} & =& \mathbb{E}\left ( \frac{1}{|\mathcal{B}|}\sum_{i\in \mathcal{B}} \nabla f_i(\widehat{x}^{s+1}_{t})-  \frac{1}{|\mathcal{B}|}\sum_{i\in \mathcal{B}}\nabla f_i(\widetilde{x}^s)  + \nabla f(\widetilde{x}^{s}) \right )
\\ &=&  \nabla f(\widehat{x}^{s+1}_{t}) -\nabla f(\widetilde{x}^s)+ \nabla f(\widetilde{x}^s)  \nonumber
\\ &=& \nabla f(\widehat{x}^{s+1}_{t})  \nonumber
\end{eqnarray}
Thus, $\widehat{v}^{s+1}_t$ is called a stochastic gradient of $f(x)$ at $\widehat{x}^{s+1}_{t}$.

Because AsySBCDVR does not use the reading  lock, the vector $\widehat{x}^{s+1}_t$ read into the local  memory may be inconsistent to the vector ${x}^{s+1}_t$ in the shared memory, which means that some components of $\widehat{x}^{s+1}_t$ are same with the ones in ${x}^{s+1}_t$, but others are different to the ones in  ${x}^{s+1}_t$.  However, we can define a set $K(t)$ of  inner iterations, such that,
  \begin{eqnarray} \label{pi_test1}
{x}^{s+1}_t = \widehat{x}^{s+1}_t + \sum_{t' \in K(t)}B_{t'}^{s+1} \Delta^{s+1}_{t'}
\end{eqnarray}
where $t' \leq t-1$, $(\Delta^{s+1}_{t'})_{\mathcal{G}_{j(t')}} = \mathcal{P}_{{\mathcal{G}_{j(t')}},\frac{\gamma}{L_{\max}}g_{j(t')}}\left (  (x_{t'}^{s+1})_{\mathcal{G}_{j(t')}} - \frac{\gamma}{L_{\max}}  \widehat{v}^{s+1}_{t',\mathcal{G}_{j(t')}}   \right ) - ({x}^{s+1}_{t'})_{\mathcal{G}_{j(t')}}$,  $(\Delta^{s+1}_{t'})_{\setminus \mathcal{G}_{j(t')}} = \textbf{0}$, $B_{t'}^{s+1}$ is   a diagonal matrix with
diagonal entries either $1$ or $0$. It is reasonable to assume that $B_{t'}^{s} \neq \textbf{0}$, $\forall {t' \in K(t)}$ (i.e., Assumption \ref{ass5}), and there exists an upper bound $\tau$ such that $\tau \geq t - \min\{t' | t' \in K(t)\}$ (i.e., Assumption \ref{ass4}).
\begin{assumption}[Non zero of $B_{t'}^{s}$]\label{ass5}
For all inner iterations $t$ in AsySBCDVR, $\forall {t' \in K(t)}$, we have that $B_{t'}^{s} \neq \textbf{0}$.
 \end{assumption}

\begin{assumption}[Bound of delay]\label{ass4}
There exists a  upper bound $\tau$ such that $\tau \geq t - \min\{t' | t' \in K(t)\}$ for all inner iterations $t$ in AsySBCDVR.
 \end{assumption}

\begin{algorithm}
\renewcommand{\algorithmicrequire}{\textbf{Input:}}
\renewcommand{\algorithmicensure}{\textbf{Output:}}
\caption{ Asynchronous Stochastic  Block Coordinate Descent  with  Variance Reduction  (AsySBCDVR)}
\begin{algorithmic}[1]
\REQUIRE $\gamma$,  $S$, and $m$.
\ENSURE $x^{S}$.
 \STATE  Initialize  $x^0 \in \mathbb{R}^d$, $p$ threads.
\FOR{$s=0,1,2,S-1$}
\STATE $\widetilde{x}^s\leftarrow x^s$
\STATE \textit{All threads parallelly} compute the full gradient $\nabla f(\widetilde{x}^s) = \frac{1}{l}\sum_i^l \nabla f_i(\widetilde{x}^s)$
\STATE \textit{For each thread}, do:
\FOR{$t=0,1,2,m-1$}
\STATE Randomly sample a mini-batch $\mathcal{B}$ from $\{1, ...,l\}$ with equal probability.
\STATE Randomly choose a block  $j(t)$ from $\{1, ...,k\}$ with equal probability.
\STATE Compute $ \widehat{v}^{s+1}_{\mathcal{G}_{j(t)}} = \frac{1}{|\mathcal{B}|}\sum_{i\in \mathcal{B}} \nabla_{\mathcal{G}_{j(t)}} f_i(\widehat{x}^{s+1}_{t})-\frac{1}{|\mathcal{B}|}\sum_{i\in \mathcal{B}}\nabla_{\mathcal{G}_{j(t)}} f_i(\widetilde{x}^s)  + \nabla_{\mathcal{G}_{j(t)}} f(\widetilde{x}^{s}) $.
 \STATE   $(x_{t+1}^{s+1})_{\mathcal{G}_{j(t)}} \leftarrow \mathcal{P}_{{\mathcal{G}_{j(t)}},\frac{\gamma}{L_{\max}}g_{j(t)}}\left (  (x_{t}^{s+1})_{\mathcal{G}_{j(t)}} - \frac{\gamma}{L_{\max}}  \widehat{v}^{s+1}_{t,\mathcal{G}_{j(t)}}   \right )$. %\COMMENT{Atomic writing}

\STATE $(x_{t+1}^{s+1})_{\setminus \mathcal{G}_{j(t)}} \leftarrow (x_{t}^{s+1})_{\setminus \mathcal{G}_{j(t)}} $.
\ENDFOR
\STATE ${x}^{s+1}\leftarrow x_{m}^{s+1}$
\ENDFOR
\end{algorithmic}
\label{algorithm2}
\end{algorithm}

\section{Convergence Analysis} \label{convergence_analysis}
In this section, we follow the analysis of \citep{liu2015asynchronous} and prove the convergence rate of AsyDSCDVR (Theorem \ref{theorem2}). Specifically, AsySBCDVR achieves a linear convergence rate when the function $f$ is with the optimal strong
convexity property, and a sublinear rate  when $f$ is with the  general convexity.

Before providing  the theoretical analysis, we   give the definitions of $\widehat{x}_{t,t'+1}^s$, ${\overline{x}}_{t+1}^{s}$ and the explanation of $x^s_t$ used in the analysis as follows.

 \begin{enumerate}
\item \textbf{$\widehat{x}_{t,t'}^s$:} Assume the indices in  $ K(t)$ are sorted in the increasing order, we use $K(t)_{t'}$ to denote the $t'$-th index in $K(t)$. For $t'= 0,1,\cdots,|K(t)|$, we define
\begin{eqnarray}
\widehat{x}_{t,t'}^s =\widehat{x}_{t}^s+ \sum_{t''=1}^{t'} \left ( B_{K(t)_{t''}}^{s} \Delta^{s}_{K(t)_{t''}} \right ) =\widehat{x}_{t}^s+ \sum_{t''=1}^{t'} \left ( x_{K(t)_{t''}+1} - x_{K(t)_{t''}} ) \right )
\end{eqnarray}
Thus, we have that
\begin{eqnarray}
\widehat{x}_{t}^s & = & \widehat{x}_{t,0}^s
\\ {x}_{t}^s & = & \widehat{x}_{t,|K(t)|}^s
\\ {x}_{t}^s - \widehat{x}_{t}^s & = & \sum_{t''=0}^{|K(t)|-1} \left ( B_{K(t)_{t''}}^{s+1} \Delta^{s+1}_{K(t)_{t''}} \right ) = \sum_{t'=0}^{|K(t)|-1} \left ( \widehat{x}_{t,t'+1}^s - \widehat{x}_{t,t'}^s \right )
\\ \nabla f(x_t^s) -\nabla f(\widehat{x}_t^s ) &=& \sum_{t'=0}^{ |K(t)|-1} \left ( \nabla f(\widehat{x}_{t,t'+1}^s) -\nabla f(\widehat{x}_{t,t'} ) \right )
\end{eqnarray}

\item  \textbf{${\overline{x}}_{t+1}^{s}$:} ${\overline{x}}_{t+1}^{s}$ is defined as:
\begin{eqnarray} \label{equ_section4_9}
{\overline{x}}_{t+1}^{s} \stackrel{\rm def}{=}  \mathcal{P}_{\frac{\gamma}{L_{\max}}g}\left ( x_t^{s} - \frac{\gamma}{L_{\max}} \widehat{v}_{t}^{s}   \right )
\end{eqnarray}
Based on (\ref{equ_section4_9}), it is easy to verify that $({\overline{x}}_{t+1}^{s})_{\mathcal{G}_{j(t)}} =(x_{t+1}^{s+1})_{\mathcal{G}_{j(t)}}$. Thus, we have $\mathbb{E}_{j(t)} ({x}_{t+1}^{s} - {x}_{t}^{s}) = \frac{1}{k} \left ( {\overline{x}}_{t+1}^{s} - {x}_{t}^{s}  \right )$. It means that $  {\overline{x}}_{t+1}^{s} - {x}_{t}^{s}$ captures the expectation  of ${x}_{t+1}^{s} - {x}_{t}^{s}$.

\item \textbf{$x^s_t$:} As mentioned previously, AsySBCDVR does not use any locks in the reading and writing. Thus,  in the line 10 of Algorithm \ref{algorithm2}, $x^s_t$ (left side of `$\leftarrow$') updated in the shared memory may be inconsistent with the idea one (right side of `$\leftarrow$') computed by the proximal operator. In the analysis, we use $x^s_t$ to denote the idea one computed by the proximal operator.  Same as mentioned in \citep{mania2015perturbed}, there might not be an actual time the idea ones exist in the
shared memory, except the first and last  iterates for each  outer loop. It is noted that, $x^s_0$ and $x^s_m$ are exactly what is stored in shared memory. Thus, we only consider the idea $x^s_t$ in the analysis.
\end{enumerate}
Then, we give two inequalities in  Lemma \ref{lemma0.5} and \ref{lemma0.6} respectively. Based on Lemma \ref{lemma0.5} and \ref{lemma0.6}, we   prove that $\mathbb{E}\| x_{t-1}^s - {\overline{x}}_{t}^s \|^2 \leq \rho \mathbb{E}\| x_{t}^s - {\overline{x}}_{t+1}^s\|^2 $ (Lemma \ref{lemma1}), where $\rho>1$ is a user defined parameter. Then, we prove  the monotonicity of the expectation of the objectives $\mathbb{E} F(x_{t+1}^s) \leq \mathbb{E} F(x_{t}^s)$ (Lemma \ref{lemma2}). Note that the analyses only consider the case $|\mathcal{B}|=1$  without loss of generality. The  case of $|\mathcal{B}|>1$ can be proved similarly.
%\begin{eqnarray}
%\overline{\overline{x}}_{t+1} & =&  \mathcal{P}_{\frac{\gamma}{L_{\max}g}}\left ( x_t - \frac{\gamma}{L_{\max}} \mathbb{E} v_{t+1}   \right )
%\\ & =& \mathcal{P}_{\frac{\gamma}{L_{\max}g}}\left ( x_t - \frac{\gamma}{L_{\max}}  \nabla f(\widehat{x}_t)   \right )         \nonumber
%\end{eqnarray}
%\begin{eqnarray}
%\mathbb{E} \left ( x_{t+1} -x_{t} \right ) = \frac{1}{d}\sum_{j=1}^{d}\left ( \mathbb{E} {\overline{x}}_{t+1} -  x_{t} \right ) e_j =\frac{1}{d}\sum_{j=1}^{d}\left ( \overline{\overline{x}}_{t+1} -  x_{t} \right ) e_j= \frac{1}{d} \left ( \overline{\overline{x}}_{t+1} -  x_{t} \right )
%\end{eqnarray}

\begin{lemma} \label{lemma0.5} For $\left \| \nabla f(x_t^s) -\nabla f(\widehat{x}_t^s ) \right \|$ in each iteration of AsySBCDVR, we have its upper bound as
\begin{eqnarray}
\left \| \nabla f(x_t^s) -\nabla f(\widehat{x}_t^s ) \right \| \leq L_{res} \sum_{t' \in K(t)} \left \|  \Delta^{s}_{t'} \right \|
\end{eqnarray}
\end{lemma}
\begin{proof}
Based on , we have that
\begin{eqnarray}
&&\left \| \nabla f(x_t^s) -\nabla f(\widehat{x}_t^s ) \right \|  =   \left \| \sum_{t'=0}^{ |K(t)|-1}  \nabla f(\widehat{x}_{t,t'+1}^s) -\nabla f(\widehat{x}_{t,t'}^s ) \right \|
\\  & \leq & \nonumber \sum_{t'=0}^{ |K(t)|-1} \left \|   \nabla f(\widehat{x}_{t,t'+1}^s) -\nabla f(\widehat{x}_{t,t'}^s ) \right \|
 \leq  \nonumber L_{res} \sum_{t'=0}^{ |K(t)|-1} \left \|   \widehat{x}_{t,t'+1}^s - \widehat{x}_{t,t'}^s  \right \|
\\  & = & \nonumber L_{res} \sum_{t'=0}^{ |K(t)|-1} \left \|  B^s_{K(t)_{t'}} \Delta^s_{K(t)_{t'}}  \right \|
 \leq  \nonumber L_{res} \sum_{t'=0}^{ |K(t)|-1} \left \|  B^s_{K(t)_{t'}} \right \| \left \| \Delta^s_{K(t)_{t'}}  \right \|
\leq  \nonumber L_{res} \sum_{t' \in K(t)} \left \|  \Delta^{s}_{t'} \right \|
\end{eqnarray}
This completes the proof.
\end{proof}

\begin{lemma} \label{lemma0.6}  In each iteration of AsySBCDVR, $\forall x$, we have the following inequality.
\begin{eqnarray}
\label{lemmaequ0.05}
 \left \langle  (\widehat{v}^s_t)_{\mathcal{G}_{j(t)}} + \frac{L_{\max}}{\gamma} \Delta^{s}_{t}, (x_{t+1}^s - x  )_{\mathcal{G}_{j(t)}}\right  \rangle   + g_{\mathcal{G}_{j(t)}}\left ( (x_{t+1}^s)_{\mathcal{G}_{j(t)}} \right )    - g_{\mathcal{G}_{j(t)}}\left ( (x)_{\mathcal{G}_{j(t)}} \right ) \leq 0
\end{eqnarray}
\end{lemma}
\begin{proof} The problem solved in lines 8 of Algorithm \ref{algorithm2} is as follows
\begin{eqnarray}
\label{lemmaequ0.1}
x_{t+1}^s = \arg \min_{x} && \left \langle  (\widehat{v}^s_t)_{\mathcal{G}_{j(t)}}, (x - x_{t}^s)_{\mathcal{G}_{j(t)}}\right  \rangle  + \frac{L_{\max}}{2\gamma} \left \| (x - x_{t}^s)_{\mathcal{G}_{j(t)}} \right \|^2
\\  && \nonumber +  g_{\mathcal{G}_{j(t)}}\left ( (x)_{\mathcal{G}_{j(t)}} \right )
\\  \nonumber s.t. &&  x_{\setminus \mathcal{G}_{j(t)}} = (x_{t}^{s})_{\setminus \mathcal{G}_{j(t)}}
\end{eqnarray}
If $x_{t+1}^s$ is the solution of (\ref{lemmaequ0.1}), the solution of  optimization problem (\ref{lemmaequ0.2}) is also $x_{t+1}^s$ according to the   subdifferential version of Karush-Kuhn-Tucker (KKT) conditions \citep{ruszczynski2006nonlinear}.
\begin{eqnarray}
\label{lemmaequ0.2}
P(x) = \min_{x} && \left \langle  (\widehat{v}^s_t)_{\mathcal{G}_{j(t)}} + \frac{L_{\max}}{\gamma} \left ( x_{t+1}^s - x_{t}^s \right )_{\mathcal{G}_{j(t)}} , (x - x_{t}^s)_{\mathcal{G}_{j(t)}}\right  \rangle  +  g_{\mathcal{G}_{j(t)}}\left ( (x)_{\mathcal{G}_{j(t)}} \right )
\\  \nonumber s.t. &&  x_{\setminus \mathcal{G}_{j(t)}} = (x_{t}^{s})_{\setminus \mathcal{G}_{j(t)}}
\end{eqnarray}
 Thus, we have that $P(x) \geq P(x_{t+1}^s) $, $\forall x$, which leads to (\ref{lemmaequ0.05}). This completes the proof.
\end{proof}

\begin{lemma} \label{lemma1}  Let $\rho$ be a constant that satisfies $\rho> 1$, and define the quantities $\theta_1 =   \frac{\rho^{\frac{1}{2}} - \rho^{\frac{\tau+1}{2}}}{1-\rho^{\frac{1}{2}}}
$ and $\theta_2 =   \frac{\rho^{\frac{1}{2}}  - \rho^{\frac{m}{2}}}{1-\rho^{\frac{1}{2}}}
$.
%\begin{eqnarray}
%\theta = \frac{\rho^{\frac{\tau+1}{2}} - \rho^{\frac{1}{2}}}{ \rho^{\frac{1}{2}} -1}, \ \ \theta' = \frac{\rho^{\tau+1}  -\rho }{\rho -1}, \ \ \psi = 1+ %\frac{\tau \theta'}{d} + \frac{2\Lambda \theta}{\sqrt{d}}
%\end{eqnarray}
Suppose the nonnegative steplength parameter  $\gamma>0$ satisfies $\gamma \leq \min \left \{ \frac{k^{1/2}(1-\rho^{-1})-4}{4\left ( \Lambda_{res} \left ( 1+ \theta_1 \right )  + \Lambda_{nor} \left ( 1+ \theta_2 \right )  \right )}, \frac{{ k^{1/2}}}{\frac{1}{2}k^{1/2} + {2  \Lambda_{nor} \theta_2 +\Lambda_{res} \theta_1 }} \right \}$, we have
\begin{eqnarray} \label{equ_lema1_con}
\mathbb{E}\| x_{t-1}^s - {\overline{x}}_{t}^s \|^2 \leq \rho \mathbb{E}\| x_{t}^s - {\overline{x}}_{t+1}^s\|^2 \end{eqnarray}
\end{lemma}
\begin{proof}
According to (A.8) in \cite{liu2015asynchronous}, we have
\begin{eqnarray} \label{equ_lema1_1} &&\|x_{t-1}^s -\overline{x}_{t}^s\|^2 - \|x_{t}^s -\overline{x}_{t+1}^s\|^2
 \leq  2  \| x_{t-1}^s -\overline{x}_{t}^s \| \| x_{t}^s -\overline{x}_{t+1}^s - x_{t-1}^s + \overline{x}_{t}^s \|
\end{eqnarray}
The second part in the right half side of (\ref{equ_lema1_1}) is bound as follows if $\mathcal{B}=\{i_t\}$ and ${J}(t)=\{j(t)\}$.
\begin{eqnarray} \label{equ_lema1_2}
&& \| x_{t}^s -\overline{x}_{t+1}^s - x_{t-1}^s + \overline{x}_{t}^s \|
\\ &= & \nonumber \left  \| x_{t}^s -\mathcal{P}_{\frac{\gamma}{L_{\max}}g}\left ( x_t^{s} - \frac{\gamma}{L_{\max}} \widehat{v}_{t}^{s}   \right ) - x_{t-1}^s + \mathcal{P}_{\frac{\gamma}{L_{\max}}g}\left ( x_{t-1}^{s} - \frac{\gamma}{L_{\max}} \widehat{v}_{t-1}^{s}   \right ) \right \|
\\ & \leq & \nonumber \| x_{t}^s -   x_{t-1}^s \| + \left  \| \mathcal{P}_{\frac{\gamma}{L_{\max}}g}\left ( x_t^{s} - \frac{\gamma}{L_{\max}} \widehat{v}_{t}^{s}   \right )    - \mathcal{P}_{\frac{\gamma}{L_{\max}}g}\left ( x_{t-1}^{s}   - \frac{\gamma}{L_{\max}} \widehat{v}_{t-1}^{s}   \right ) \right  \|
\\ &\leq& \nonumber 2\| x_{t}^s -   x_{t-1}^s \| + \frac{\gamma}{L_{\max}} \left  \| \widehat{v}_{t}^{s} -  \widehat{v}_{t-1}^{s}  \right  \|
\\ &=& \nonumber 2\| x_{t}^s -   x_{t-1}^s \| +
  \\  && \nonumber \frac{\gamma}{L_{\max}} \left  \| \nabla f_{i_t}(\widehat{x}^{s}_{t})- \nabla f_{i_t}(\widetilde{x}^s)  +  \nabla f(\widetilde{x}^{s}) - \nabla f_{i_{t-1}}(\widehat{x}^{s}_{t-1})+ \nabla f_{i_{t-1}}(\widetilde{x}^s) - \nabla f(\widetilde{x}^{s}) \right  \|
\\ &=& \nonumber 2\| x_{t}^s -   x_{t-1}^s \| +  \frac{\gamma}{L_{\max}} \left  \| \nabla f_{i_t}(\widehat{x}^{s}_{t})  - \nabla f_{i_t}(\widetilde{x}^s) - \nabla f_{i_{t-1}}(\widehat{x}^{s}_{t-1}) + \nabla f_{i_{t-1}}(\widetilde{x}^s)  \right  \|
 \\ &\leq& \nonumber 2\| x_{t}^s -   x_{t-1}^s \| + \frac{\gamma}{L_{\max}} \left  \| \nabla f_{i_t}(\widehat{x}^{s}_{t}) - \nabla f_{i_t}({x}^{s}_{t}) + \nabla f_{i_t}({x}^{s}_{t})  - \nabla f_{i_t}(\widetilde{x}^s)  \right  \|
    \\  && \nonumber + \frac{\gamma}{L_{\max}} \left  \|  \nabla f_{i_{t-1}}(\widehat{x}^{s}_{t-1}) - \nabla f_{i_{t-1}}({x}^{s}_{t-1})  + \nabla f_{i_{t-1}}({x}^{s}_{t-1}) - \nabla f_{i_{t-1}}(\widetilde{x}^s) \right  \|
  \\ &\leq& \nonumber 2\| x_{t}^s -   x_{t-1}^s \| + \frac{\gamma}{L_{\max}} \left  \| \nabla f_{i_t}(\widehat{x}^{s}_{t}) - \nabla f_{i_t}({x}^{s}_{t})  \right  \| + \frac{\gamma}{L_{\max}} \left  \| \nabla f_{i_t}({x}^{s}_{t})  - \nabla f_{i_t}(\widetilde{x}^s) \right  \|
    \\  && \nonumber  + \frac{\gamma}{L_{\max}} \left  \| \nabla f_{i_{t-1}}(\widehat{x}^{s}_{t-1}) - \nabla f_{i_{t-1}}({x}^{s}_{t-1}) \right  \|  + \frac{\gamma}{L_{\max}} \left  \| \nabla f_{i_{t-1}}({x}^{s}_{t-1}) - \nabla f_{i_{t-1}}(\widetilde{x}^s)   \right  \|
 \\ \nonumber &\leq& 2\| x_{t}^s -   x_{t-1}^s \| +
 {\gamma \Lambda_{res}} \left (  \sum_{t' \in K(t-1)} \| \Delta^{s}_{t'}  \|+  \sum_{t' \in K(t)} \| \Delta^{s}_{t'}  \| \right )
  \\  && \nonumber + {\gamma \Lambda_{nor}} \left ( \| x_{t}^{s} - \widetilde{x}^s \|  + \| x_{t-1}^{s} - \widetilde{x}^s \| \right )
% \\ \nonumber &\leq&  2 \| x_{t}^s -   x_{t-1}^s \| +
%2 \gamma  \left ( \Lambda_{res}\sum_{t' = t-1 -\tau }^{t-1} \| \Delta^s_{t'} \|  + \Lambda_{nor}\sum_{t' = 0 }^{t-1} \|  B^s_{t'} \Delta^s_{t'}  \| \right )
\\ \nonumber &\leq&  2 \| x_{t}^s -   x_{t-1}^s \| +
2 \gamma  \left ( \Lambda_{res}\sum_{t' = t-1 -\tau }^{t-1} \| \Delta^s_{t'} \|  + \Lambda_{nor}\sum_{t' = 0 }^{t-1} \|  \Delta^s_{t'}  \| \right )
\end{eqnarray}
where the first inequality use the nonexpansive property of $\mathcal{P}_{\frac{\gamma}{L_{\max}}g}$, the fifth inequality use A.7 of \cite{liu2015asynchronous}, the sixth inequality comes from $\| x_{t}^s - \widetilde{x}^s \| =\| \sum_{t' = 0 }^{t-1}  \Delta^s_{t'}  \| \leq \sum_{t' = 0 }^{t-1}\| \Delta^s_{t'} \|$.

If $t=1$, we have that $K(0) = \emptyset$ and $K(1) \subseteq \{ 0 \}$. Thus, according to (\ref{equ_lema1_2}), we have
\begin{eqnarray} \label{equ_lema1_2.5}
 \| x_{1}^s -\overline{x}_{2}^s - x_{0}^s + \overline{x}_{1}^s \| \leq  2 \| x_{1}^s -   x_{0}^s \| +
2\gamma \left ( \Lambda_{res} + \Lambda_{nor} \right ) \| \Delta^s_{0} \|
\end{eqnarray}
Substituting (\ref{equ_lema1_2.5}) into (\ref{equ_lema1_1}), and takeing expectations, we have
\begin{eqnarray} \label{equ_lema1_2.6}
 &&\mathbb{E}  \|x_{0}^s -\overline{x}_{1}^s\|^2 - \mathbb{E} \|x_{1}^s -\overline{x}_{2}^s\|^2
  \leq   2 \mathbb{E}  \left ( \| x_{0}^s -\overline{x}_{1}^s \| \| x_{1}^s -\overline{x}_{2}^s - x_{0}^s + \overline{x}_{1}^s \|  \right )
 \\ &\leq&  \nonumber 4 \mathbb{E} \left (\| x_{0}^s -\overline{x}_{1}^s \|  \| x_{1}^s -   x_{0}^s \| \right ) +  4\gamma \left ( \Lambda_{res} + \Lambda_{nor} \right )  \mathbb{E} \left ( \| x_{0}^s -\overline{x}_{1}^s \|  \| \Delta^s_{0} \|  \right )
 \\ &\leq&  \nonumber　4 k^{-\frac{1}{2}} \mathbb{E} \left (\| x_{0}^s -\overline{x}_{1}^s \|^2　\right ) +   4\gamma \left ( \Lambda_{res} + \Lambda_{nor} \right )  \mathbb{E} \left ( \| x_{0}^s -\overline{x}_{1}^s \|  \| \Delta^s_{0} \|  \right )
\end{eqnarray}
where the last inequality uses A.13 in \citep{liu2015asynchronous}. Further, we have the upper bound of $\mathbb{E} \left ( \| x_{t}^s -\overline{x}_{t+1}^s \|  \| \Delta^s_{t} \|  \right )$ as
\begin{eqnarray} \label{equ_lema1_2.7}
 && \mathbb{E} \left ( \| x_{t}^s -\overline{x}_{t+1}^s \|  \| \Delta^s_{t} \|  \right ) \leq  \frac{1}{2}\mathbb{E} \left ( k^{-\frac{1}{2}} \| x_{t}^s -\overline{x}_{t+1}^s \|^2 + k^{ \frac{1}{2}} \| \Delta^s_{t} \|^2  \right )
 \\ &= &  \nonumber \frac{1}{2} \mathbb{E} \left ( k^{-\frac{1}{2}} \| x_{t}^s -\overline{x}_{t+1}^s \|^2 + k^{\frac{1}{2}} \mathbb{E}_{j(t)} \| \Delta^s_{t} \|^2  \right ) = \frac{1}{2}
\mathbb{E} \left ( k^{-\frac{1}{2}} \| x_{t}^s -\overline{x}_{t+1}^s \|^2 + k^{- \frac{1}{2}}\mathbb{E} \| x_{t}^s -\overline{x}_{t+1}^s \|^2  \right )
\\ &= &  \nonumber  k^{- \frac{1}{2}}\mathbb{E} \| x_{t}^s -\overline{x}_{t+1}^s \|^2
\end{eqnarray}
Substituting (\ref{equ_lema1_2.7}) into (\ref{equ_lema1_2.6}), we have
\begin{eqnarray} \label{equ_lema1_2.8}
 &&\mathbb{E}  \|x_{0}^s -\overline{x}_{1}^s\|^2 - \mathbb{E} \|x_{1}^s -\overline{x}_{2}^s\|^2
  \leq   　 k^{-\frac{1}{2}} \left ( 4 + 4\gamma \left ( \Lambda_{res} + \Lambda_{nor} \right ) \right ) \mathbb{E} \left (\| x_{0}^s -\overline{x}_{1}^s \|^2　\right )
\end{eqnarray}
which implies that
\begin{eqnarray} \label{equ_lema1_2.9}
 &&\mathbb{E}  \|x_{0}^s -\overline{x}_{1}^s\|^2
  \leq \left ( 1 - \frac{ 4 + 4\gamma \left ( \Lambda_{res} + \Lambda_{nor} \right ) }{\sqrt{k}}  \right )^{-1}  \mathbb{E} \|x_{1}^s -\overline{x}_{2}^s\|^2
  \leq \rho \mathbb{E} \|x_{1}^s -\overline{x}_{2}^s\|^2
\end{eqnarray}
where the last inequality  follows from . Thus, we have (\ref{equ_lema1_con}) for $t=1$.
\begin{eqnarray} \label{equ_lema1_2.10}
\rho^{-1} \leq  1 - \frac{ 4 + 4\gamma \left ( \Lambda_{res} + \Lambda_{nor} \right ) }{\sqrt{k}}   \Leftrightarrow \gamma \leq \frac{k^{1/2}(1-\rho^{-1})-4}{4\left ( \Lambda_{res}   + \Lambda_{nor}   \right )}
\end{eqnarray}

Next, we consider the cases for $t>1$. For $t-1-\tau \leq t' \leq t-1$ and any $\beta>0$,  we have
\begin{eqnarray} \label{equ_lema1_2.11}
&& \mathbb{E} \left ( \| x_{t}^s -\overline{x}_{t+1}^s \|  \| \Delta^s_{t'} \|  \right ) \leq
\frac{1}{2}\mathbb{E} \left ( k^{-\frac{1}{2}} \beta^{-1} \| x_{t}^s -\overline{x}_{t+1}^s \|^2 + k^{ \frac{1}{2}} \beta \| \Delta^s_{t'} \|^2  \right )
 \\ &= &  \nonumber \frac{1}{2} \mathbb{E} \left (  k^{-\frac{1}{2}} \beta^{-1}  \| x_{t}^s -\overline{x}_{t+1}^s \|^2 + k^{ \frac{1}{2}} \beta \mathbb{E}_{j(t)} \| \Delta^s_{t'} \|^2  \right )
\\ &= &  \nonumber  \frac{1}{2}
\mathbb{E} \left ( k^{-\frac{1}{2}} \beta^{-1} \| x_{t}^s -\overline{x}_{t+1}^s \|^2 + k^{-\frac{1}{2}} \beta \mathbb{E} \| x_{t'}^s -\overline{x}_{t'+1}^s \|^2  \right )
\\ & \leq &  \nonumber  \frac{1}{2}
\mathbb{E} \left ( k^{-\frac{1}{2}} \beta^{-1} \| x_{t}^s -\overline{x}_{t+1}^s \|^2 + k^{-\frac{1}{2}} \rho^{t-t'} \beta \mathbb{E} \| x_{t}^s -\overline{x}_{t+1}^s \|^2  \right )
\\ & \stackrel{\beta =\rho^{\frac{t'-t}{2}} }{\leq } &  \nonumber k^{-\frac{1}{2}} \rho^{\frac{t-t'}{2}} \mathbb{E}  \| x_{t}^s -\overline{x}_{t+1}^s \|^2
\end{eqnarray}

 We assume that (\ref{equ_lema1_con}) holds $\forall t' <t$. By substituting (\ref{equ_lema1_2}) into (\ref{equ_lema1_1}) and taking expectation on both sides of (\ref{equ_lema1_1}), we can have
\begin{eqnarray} \label{equ_lema1_3}
&&\mathbb{E} \left ( \|x_{t-1}^s -\overline{x}_{t}^s\|^2 - \|x_{t}^s -\overline{x}_{t+1}^s\|^2  \right )
\\ &\leq& 2 \mathbb{E} \left ( \| x_{t-1}^s -\overline{x}_{t}^s \| \| x_{t}^s -\overline{x}_{t+1} - x_{t-1}^s + \overline{x}_{t}^s \| \right )  \nonumber
\\ &\leq& \nonumber   2 \mathbb{E} \left (  \| x_{t-1}^s -\overline{x}_{t}^s \| \left ( 2 \| x_{t}^s -   x_{t-1}^s \| +
2 \gamma  \left ( \Lambda_{res}\sum_{t' = t-1 -\tau }^{t-1} \| \Delta^s_{t'} \|  + \Lambda_{nor}\sum_{t' = 0 }^{t-1} \|  \Delta^s_{t'}  \| \right ) \right ) \right )
\\ & = & \nonumber  4 \mathbb{E} \left (  \| x_{t-1}^s -\overline{x}_{t}^s \|   \| x_{t-1}^s -\overline{x}_{t}^s \|  \right ) +
\\ && \nonumber  4  \gamma   \mathbb{E} \left ( \Lambda_{res} \sum_{t' = t-1 -\tau }^{t-1} \| x_{t-1}^s -\overline{x}_{t}^s \|  \|  \Delta^s_{t'} \| + \Lambda_{nor} \sum_{t' = 0 }^{t-1} \| x_{t-1}^s -\overline{x}_{t}^s \| \|  \Delta^s_{t'} \| \right )
\\ &\leq& \nonumber  4 k^{-1/2} \mathbb{E}\left (  \| x_{t-1}^s -\overline{x}_{t}^s \|^2 \right )  +
 \\ && \nonumber  4  \gamma  k^{-1/2} \mathbb{E}\left (  \| x_{t-1}^s -\overline{x}_{t}^s \|^2 \right ) \cdot
  \left (\Lambda_{res} \sum_{t' = t-1 -\tau }^{t-1} \rho^{\frac{t-1-t'}{2}} + \Lambda_{nor}\sum_{t' = 0 }^{t-1}\rho^{\frac{t-1-t'}{2}}   \right )
\\ &=& \nonumber  (4+4 \gamma \left  ( \Lambda_{res} +  \Lambda_{nor} \right ) )k^{-1/2} \mathbb{E}\left (  \| x_{t-1}^s -\overline{x}_{t}^s \|^2 \right )  +
\\ && \nonumber  4  \gamma  k^{\frac{-1}{2}}  \mathbb{E}\left (  \| x_{t-1}^s -\overline{x}_{t}^s \|^2 \right ) \left ( \Lambda_{res }\sum_{t' = 1 }^{\tau} \rho^{\frac{t'}{2}} + \Lambda_{nor}\sum_{t' = 1 }^{t-1}\rho^{\frac{t'}{2}}   \right )
\\ & = & \nonumber k^{-1/2} \mathbb{E}\left (  \| x_{t-1}^s -\overline{x}_{t}^s \|^2 \right )  \left (4 +4 \gamma \Lambda_{res}   \left ( 1+ \frac{\rho^{\frac{1}{2}} - \rho^{\frac{\tau+1}{2}}}{1-\rho^{\frac{1}{2}}} \right ) + 4 \gamma \Lambda_{nor} \left ( 1+ \frac{\rho^{\frac{1}{2}}  - \rho^{\frac{m}{2}}}{1-\rho^{\frac{1}{2}}} \right )  \right )
\\ & = & \nonumber k^{-1/2} \mathbb{E}\left (  \| x_{t-1}^s -\overline{x}_{t}^s \|^2 \right )  \cdot
  \left (4+4 \gamma \left ( \Lambda_{res} \left ( 1+ \theta_1 \right )  + \Lambda_{nor} \left ( 1+ \theta_2 \right )  \right ) \right )
\end{eqnarray}
where the third inequality uses (\ref{equ_lema1_2.11}).
Based on (\ref{equ_lema1_3}), we have that
\begin{eqnarray} \label{equ_lema1_4}
&&\mathbb{E} \left ( \|x_{t-1}^s -\overline{x}_{t}^s\|^2 \right )
\\ & \leq & \nonumber \left ( 1- k^{-1/2} \left  (4+4 \gamma \left ( \Lambda_{res} \left ( 1+ \theta_1 \right )  + \Lambda_{nor} \left ( 1+ \theta_2 \right )  \right ) \right ) \right )^{-1}
 \cdot \mathbb{E} \left (  \|x_{t}^s -\overline{x}_{t+1}^s\|^2  \right )
\\ & \leq & \nonumber \rho \mathbb{E} \left (  \|x_{t}^s -\overline{x}_{t+1}^s\|^2  \right )
\end{eqnarray}
where the last inequality follows from
\begin{eqnarray} \label{equ_lema1_5}  \rho^{-1} \leq   1- k^{-1/2} \left  (4+4 \gamma \left ( \Lambda_{res} \left ( 1+ \theta_1 \right )  + \Lambda_{nor} \left ( 1+ \theta_2 \right )  \right ) \right )
\\ \nonumber \Leftrightarrow  \gamma \leq \frac{k^{1/2}(1-\rho^{-1})-4}{4\left ( \Lambda_{res} \left ( 1+ \theta_1 \right )  + \Lambda_{nor} \left ( 1+ \theta_2 \right )  \right )}
\end{eqnarray}
 This completes the proof.
\end{proof}

\begin{lemma} \label{lemma2} Let $\rho$ be a constant that satisfies $\rho> 1$, and define the quantities $\theta_1 =   \frac{\rho^{\frac{1}{2}} - \rho^{\frac{\tau+1}{2}}}{1-\rho^{\frac{1}{2}}}
$ and $\theta_2 =   \frac{\rho^{\frac{1}{2}}  - \rho^{\frac{m}{2}}}{1-\rho^{\frac{1}{2}}}
$. Suppose the nonnegative steplength parameter  $\gamma>0$ satisfies $\gamma \leq \min \left \{ \frac{k^{1/2}(1-\rho^{-1})-4}{4\left ( \Lambda_{res} \left ( 1+ \theta_1 \right )  + \Lambda_{nor} \left ( 1+ \theta_2 \right )  \right )}, \frac{{ k^{1/2}}}{\frac{1}{2}k^{1/2} + {2  \Lambda_{nor} \theta_2 +\Lambda_{res} \theta_1 }} \right \}$. The expectation of the objective function $\mathbb{E}  F(x_{t}^s)$ is monotonically decreasing, i.e., $\mathbb{E} F(x_{t+1}^s) \leq \mathbb{E} F(x_{t}^s)$.
\end{lemma}
\begin{proof} Take expectation  $F(x_{t+1}^s)$  on ${j(t)}$, we have that
\begin{eqnarray} \label{equ_lema1_5}
&& \mathbb{E}_{j(t)}  F(x_{t+1}^s) = \mathbb{E}_{j(t)} \left (   f(x_{t}^s +  \Delta_t^s) + g(x_{t+1}^s)  \right )
 \\ \nonumber &\leq&  \mathbb{E}_{j(t)} \left ( f(x_{t}^s ) + \left  \langle  \nabla_{\mathcal{G}_{j(t)}} f({x}^{s}_t) , ( \Delta_t^s)_{\mathcal{G}_{j(t)}}\right  \rangle + \frac{L_{\max}}{2} \left \| ( \Delta_t^s)_{\mathcal{G}_{j(t)}} \right \|^2  \right .
\\ && \nonumber \left .  + g_{\mathcal{G}_{j(t)}}\left ( (x_{t+1}^s)_{\mathcal{G}_{j(t)}} \right ) + \sum_{j' \neq j(t)}g_{\mathcal{G}_{j'}} \left ( (x_{t+1}^s)_{\mathcal{G}_{j'}}  \right )  \right )
\\ \nonumber &= &  f(x_{t}^s ) +   \frac{k-1}{k} g(x_{t}^s) + \mathbb{E}_{j(t)} \left (\left  \langle  \nabla_{\mathcal{G}_{j(t)}} f({x}^{s}_t) , (\Delta_t^s)_{\mathcal{G}_{j(t)}}\right  \rangle + \frac{L_{\max}}{2} \left \| ( \Delta_t^s)_{\mathcal{G}_{j(t)}} \right \|^2  \right .
\\ && \nonumber \left .  + g_{\mathcal{G}_{j(t)}}\left ( (x_{t+1}^s)_{\mathcal{G}_{j(t)}} \right )  \right )
\\ \nonumber &=&  F(x_{t}^s ) +   \mathbb{E}_{j(t)} \left (\left  \langle  (\widehat{v}^s_t)_{\mathcal{G}_{j(t)}} , ( \Delta_t^s)_{\mathcal{G}_{j(t)}}\right  \rangle + \frac{L_{\max}}{2} \left \| ( \Delta_t^s)_{\mathcal{G}_{j(t)}} \right \|^2   + g_{\mathcal{G}_{j(t)}}\left ( (x_{t+1}^s)_{\mathcal{G}_{j(t)}} \right ) \right .
\\ && \nonumber \left .   - g_{\mathcal{G}_{j(t)}}\left ( (x_{t}^s)_{\mathcal{G}_{j(t)}} \right ) +\left  \langle  \nabla_{\mathcal{G}_{j(t)}} f({x}^{s}_t) -(\widehat{v}^s_t)_{\mathcal{G}_{j(t)}} , ( \Delta_t^s)_{\mathcal{G}_{j(t)}}\right  \rangle \right )
\\ & \leq & \nonumber F(x_{t}^s) + \mathbb{E}_{j(t)} \left (-\frac{L_{\max}}{\gamma} \left \| ( \Delta_t^s)_{\mathcal{G}_{j(t)}} \right \|^2 + \frac{L_{\max}}{2} \left \| ( \Delta_t^s)_{\mathcal{G}_{j(t)}} \right \|^2   + \right .
\\ && \nonumber \left .  +\left  \langle  \nabla_{\mathcal{G}_{j(t)}} f({x}^{s}_t) -(\widehat{v}^s_t)_{\mathcal{G}_{j(t)}} , ( \Delta_t^s)_{\mathcal{G}_{j(t)}}\right  \rangle \right )
\\ & = & \nonumber F(x_{t}^s) + \mathbb{E}_{j(t)} \left ( \left ( \frac{L_{\max}}{2} -\frac{ L_{\max}}{\gamma}  \right ) \left \| ( \Delta_t^s)_{\mathcal{G}_{j(t)}} \right \|^2    \right )
 +\mathbb{E}_{j(t)}    \left  \langle  \nabla_{\mathcal{G}_{j(t)}} f({x}^{s}_t) -(\widehat{v}^s_t)_{\mathcal{G}_{j(t)}} , ( \Delta_t^s)_{\mathcal{G}_{j(t)}}\right  \rangle
\\ & = & \nonumber F(x_{t}^s) +  \frac{L_{\max}}{k}\left ( \frac{1}{2} -\frac{ 1}{\gamma}  \right ) \| \overline{x}^s_{t+1} -x_t^s  \|^2
 +\mathbb{E}_{j(t)}    \left  \langle  \nabla_{\mathcal{G}_{j(t)}} f({x}^{s}_t) -(\widehat{v}^s_t)_{\mathcal{G}_{j(t)}} , ( \Delta_t^s)_{\mathcal{G}_{j(t)}}\right  \rangle
 \end{eqnarray}
 where the first inequality uses (\ref{coordinate_lipschitz_constant2}), and the second inequality uses (\ref{lemmaequ0.05}) in Lemma \ref{lemma0.6}.
  %the third inequality uses $\left \| ( \Delta_t^s)_{\mathcal{G}_{j(t)}} \right \| \leq \left  \| (B_t^s \Delta_t^s)_{\mathcal{G}_{j(t)}} \right \| \leq \left  \| (\Delta_t^s)_{\mathcal{G}_{j(t)}} \right \|$ based on that both of the upper and lower bounds of $ B_t^s $  are 1 \citep{grcar2010matrix}.
 Consider the expectation of the last term on the right-hand side of (\ref{equ_lema1_5}), we have
 \begin{eqnarray} \label{equ_lema1_6}
&& \mathbb{E} \left  \langle  \nabla_{\mathcal{G}_{j(t)}} f({x}^{s}_t) -(\widehat{v}^s_t)_{\mathcal{G}_{j(t)}} , (\Delta_t^s)_{\mathcal{G}_{j(t)}}\right  \rangle
\\ \nonumber &=&  \mathbb{E} \left  \langle  \nabla_{\mathcal{G}_{j(t)}} f({x}^{s}_t) - \left (\nabla f_{i_{t}}(\widehat{x}^{s}_{t}) - \nabla f_{i_{t}}(\widetilde{x}^s) + \nabla f(\widetilde{x}^{s}) \right  )_{\mathcal{G}_{j(t)}} , ( \Delta_t^s)_{\mathcal{G}_{j(t)}}\right  \rangle
\\ \nonumber &=& \mathbb{E}  \left  \langle  \nabla_{\mathcal{G}_{j(t)}} f({x}^{s}_t) -(\nabla f_{i_{t}}(\widehat{x}^{s}_{t}) -\nabla f_{i_{t}}({x}^{s}_{t}) + \nabla f_{i_{t}}({x}^{s}_{t}) - \nabla f_{i_{t}}(\widetilde{x}^s) + \nabla f(\widetilde{x}^{s}) )_{\mathcal{G}_{j(t)}} ,
(\Delta_t^s)_{\mathcal{G}_{j(t)}}\right  \rangle
\\ \nonumber &=& \mathbb{E}  \left  \langle \nabla_{\mathcal{G}_{j(t)}} f({x}^{s}_t) - \nabla_{\mathcal{G}_{j(t)}} f(\widetilde{x}^{s}), ( \Delta_t^s)_{\mathcal{G}_{j(t)}} \right  \rangle + \mathbb{E}  \left  \langle \nabla_{\mathcal{G}_{j(t)}} f_{i_{t}}({x}^{s}_t) - \nabla_{\mathcal{G}_{j(t)}} f_{i_{t}}(\widehat{x}^{s}_{t}), (\Delta_t^s)_{\mathcal{G}_{j(t)}} \right  \rangle
\\ \nonumber &&  +\mathbb{E}  \left  \langle \nabla_{\mathcal{G}_{j(t)}} f_{i_{t}}(\widetilde{x}^s) - \nabla_{\mathcal{G}_{j(t)}} f_{i_{t}}({x}^{s}_{t}), ( \Delta_t^s)_{\mathcal{G}_{j(t)}} \right  \rangle
\\ \nonumber & \leq & \mathbb{E} \left ( \left  \| \nabla_{\mathcal{G}_{j(t)}} f({x}^{s}_t) - \nabla_{\mathcal{G}_{j(t)}} f(\widetilde{x}^{s})\right  \| \left  \| \Delta_t^s \right  \| \right ) + \mathbb{E} \left ( \left  \|  \nabla_{\mathcal{G}_{j(t)}} f_{i_{t}}({x}^{s}_t) - \nabla_{\mathcal{G}_{j(t)}} f_{i_{t}}(\widehat{x}^{s}_{t})\right  \| \left  \|  \Delta_t^s \right  \|  \right )
\\ \nonumber &&  +\mathbb{E} \left ( \left  \| \nabla_{\mathcal{G}_{j(t)}} f_{i_{t}}(\widetilde{x}^s) - \nabla_{\mathcal{G}_{j(t)}} f_{i_{t}}({x}^{s}_{t})\right  \|  \left  \|  \Delta_t^s \right  \|  \right )
\\ \nonumber & = &  \frac{1}{k}\mathbb{E} \left ( \left  \| \nabla f({x}^{s}_t) - \nabla f(\widetilde{x}^{s})\right  \| \left  \|  \overline{x}^s_{t+1} -x_t^s \right  \| \right ) + \frac{1}{k}\mathbb{E} \left ( \left  \|  \nabla f_{i_{t}}({x}^{s}_t) - \nabla f_{i_{t}}(\widehat{x}^{s}_{t})\right  \| \left  \|   \overline{x}^s_{t+1} -x_t^s \right  \|  \right )
\\ \nonumber &&  +\frac{1}{k}\mathbb{E} \left ( \left  \| \nabla f_{i_{t}}(\widetilde{x}^s) - \nabla f_{i_{t}}({x}^{s}_{t})\right  \|  \left  \|  \overline{x}^s_{t+1} -x_t^s  \right  \|  \right )
\\ \nonumber & \leq & \frac{1}{k}\mathbb{E} \left ( 2L_{nor} \| {x}^s_{t} - \widetilde{x}^s  \| \| \overline{x}^s_{t+1} -x_t^s  \| + L_{res} \sum_{t' \in K(t)}  \left \|  \Delta^{s}_{t'} \right \| \| \overline{x}^s_{t+1} -x_t^s  \|\right )
\\ \nonumber & \leq &  \frac{1}{k}\mathbb{E} \left ( 2L_{nor}\sum_{t'=0}^{t-1} \|  \Delta_{t'}^s  \| \| \overline{x}^s_{t+1} -x_t^s  \| + L_{res} \sum_{t'= t-\tau}^{t-1}  \|  \Delta^{s}_{t'}  \| \| \overline{x}^s_{t+1} -x_t^s  \|\right )
\\ \nonumber & \leq &   2L_{nor}\sum_{t'=0}^{t-1} \frac{\rho^{\frac{t-t'}{2}}}{k^{3/2}} \mathbb{E} \| \overline{x}^s_{t+1} -x_t^s  \|^2
 + L_{res} \sum_{t'= t-\tau}^{t-1} \frac{\rho^{\frac{t-t'}{2}}}{k^{3/2}} \mathbb{E}  \| \overline{x}^s_{t+1} -x_t^s  \|^2
\\ \nonumber &= & k^{-3/2} \left ( 2L_{nor} \frac{\rho^{\frac{1}{2}}  - \rho^{\frac{m}{2}}}{1-\rho^{\frac{1}{2}}} + L_{res}  \frac{\rho^{\frac{1}{2}} - \rho^{\frac{\tau+1}{2}}}{1-\rho^{\frac{1}{2}}} \right )
\cdot \mathbb{E}  \| \overline{x}^s_{t+1} -x_t^s  \|^2
\\ \nonumber &=& k^{-3/2} \left (2  L_{nor} \theta_2 +L_{res} \theta_1  \right ) \mathbb{E}  \| \overline{x}^s_{t+1} -x_t^s  \|^2
 \end{eqnarray}
where  the first inequality uses the  Cauchy-Schwarz inequality \citep{callebaut1965generalization}, the third inequality uses (\ref{equdef1}) and (\ref{equdef2}), the sixth inequality uses (\ref{equ_lema1_2.11}).

  By taking expectations on both sides of (\ref{equ_lema1_5}) and substituting (\ref{equ_lema1_6}), we have
  \begin{eqnarray}
&& \mathbb{E}  F(x_{t+1}^s)
 \\ & \leq & \nonumber  \mathbb{E} F(x_{t}^s) +  \frac{L_{\max}}{k}\left ( \frac{1}{2} -\frac{ 1 }{\gamma}  \right ) \mathbb{E} \| \overline{x}^s_{t+1} -x_t^s  \|^2
 +  \mathbb{E}   \left  \langle  \nabla_{\mathcal{G}_{j(t)}} f({x}^{s}_t) -(\widehat{v}^s_t)_{\mathcal{G}_{j(t)}} , ( \Delta_t^s)_{\mathcal{G}_{j(t)}}\right  \rangle
 \\ & \leq & \nonumber\mathbb{E}  F(x_{t}^s)- \frac{1}{k} \cdot
 \left (L_{\max} \left  (\frac{ 1 }{\gamma} -  \frac{1}{2} \right ) -\frac{2  L_{nor} \theta_2 +L_{res} \theta_1 }{k^{1/2}} \right )\mathbb{E}  \| \overline{x}^s_{t+1} -x_t^s  \|^2
    \end{eqnarray}
    where $ L_{\max} \left  (\frac{ 1}{\gamma} -  \frac{1}{2} \right ) -\frac{2  L_{nor} \theta_2 +L_{res} \theta_1 }{k^{1/2}}\geq 0 $ because that $\gamma^{-1} \geq \frac{1}{2} + \frac{2  \Lambda_{nor} \theta_2 +\Lambda_{res} \theta_1 }{k^{1/2}}$.  This completes the proof.
\end{proof}

\begin{theorem} \label{theorem2} Let $\rho$ be a constant that satisfies $\rho> 1$, and define the quantity $\theta' =  \frac{\rho^{\tau+1}- \rho}{\rho-1}
$. Suppose the nonnegative steplength parameter  $\gamma>0$ satisfies $ 1-  \Lambda_{nor} \gamma - \frac{ \gamma\tau \theta'}{n} - \frac{2 (\Lambda_{res} \theta_1 +\Lambda_{nor} \theta_2 ) \gamma}{n^{1/2} } \geq 0$. If   the optimal strong convexity holds for $f$ with $l>0$, we have
\begin{eqnarray}\label{theorem_equ_2}
 \mathbb{E}F(x^s ) - F^*
\leq \frac{L_{\max }}{2\gamma} \left ( \frac{1}{1+\frac{2 m \gamma l}{2k(l \gamma +L_{\max})}} \right )^s \cdot  \left ( \| x^0 - \mathcal{P}_S( x^0) \|^2 +  \frac{2\gamma}{L_{\max}}\left (\mathbb{E}F(x^0 ) - F^* \right ) \right )
\end{eqnarray}
If $f$ is a general smooth convex function, we have
\begin{eqnarray} \label{theorem_equ_1}
 \mathbb{E}F(x^s ) - F^*
 \leq \frac{n L_{\max} \| x^0 - \mathcal{P}_S( x^0) \|^2  + 2 \gamma k \left ( F(x^0 ) - F^*   \right )}{2 \gamma k+2m \gamma s}
\end{eqnarray}
\end{theorem}
\begin{proof}We have that
\begin{eqnarray} \label{equ_lema1_7}
&& \| x_{t+1}^s - \mathcal{P}_S( x^s_{t+1}) \|^2 \leq \| x_{t+1}^s - \mathcal{P}_S( x^s_{t}) \|^2
 =   \| x_{t}^s +  \Delta_t^s - \mathcal{P}_S( x^s_{t}) \|^2
\\ & = & \nonumber  \| x_{t}^s - \mathcal{P}_S( x^s_{t}) \|^2 - \| \Delta_t^s \|^2 - 2 \langle \left (  \mathcal{P}_S( x^s_{t}) -x_{t}^s 　 - \Delta_t^s \right )_{\mathcal{G}_{j(t)}},( \Delta_t^s)_{\mathcal{G}_{j(t)}} \rangle
\\ & =  & \nonumber  \| x_{t}^s - \mathcal{P}_S( x^s_{t}) \|^2 - \|  \Delta_t^s \|^2 - 2 \langle \left (  \mathcal{P}_S( x^s_{t}) -x_{t+1}^s  \right )_{\mathcal{G}_{j(t)}},( \Delta_t^s)_{\mathcal{G}_{j(t)}} \rangle
%\\ & \leq  & \nonumber  \| x_{t}^s - \mathcal{P}_S( x^s_{t}) \|^2 - \| \Delta_t^s \|^2 - 2 \langle \left (  \mathcal{P}_S( x^s_{t}) -x_{t+1}^s  \right )_{\mathcal{G}_{j(t)}},( \Delta_t^s)_{\mathcal{G}_{j(t)}} \rangle
\\ & \leq & \nonumber   \| x_{t}^s - \mathcal{P}_S( x^s_{t}) \|^2  - \|  \Delta_t^s \|^2+
\frac{2 \gamma}{L_{\max}} \left ( \langle \left (  \mathcal{P}_S( x^s_{t}) -x_{t+1}^s 　\right )_{\mathcal{G}_{j(t)}},( \widehat{v}_t^s)_{\mathcal{G}_{j(t)}} \rangle \right ) +
\\ &  & \nonumber \frac{2 \gamma}{L_{\max}}  \left ( g_{\mathcal{G}_{j(t)}}(\mathcal{P}_S( x^s_{t})_{\mathcal{G}_{j(t)}}) -g_{j(t)}( x^s_{t+1})_{\mathcal{G}_{j(t)}})   \right )
\\ & = & \nonumber   \| x_{t}^s - \mathcal{P}_S( x^s_{t}) \|^2  - \| \Delta_t^s \|^2 +
\frac{2 \gamma}{L_{\max}}  \underbrace{ \left ( \langle \left (  \mathcal{P}_S( x^s_{t}) -x_{t}^s 　\right )_{\mathcal{G}_{j(t)}},( \widehat{v}_t^s)_{\mathcal{G}_{j(t)}} \rangle \right ) }_{T_1}  +
\\ &  & \nonumber \frac{2 \gamma}{L_{\max}} \underbrace{ \left (\langle \left (  \Delta_{t}^s　\right )_{\mathcal{G}_{j(t)}},( \widehat{v}_t^s)_{\mathcal{G}_{j(t)}} \rangle \right ) }_{T_2} +
 \frac{2 \gamma}{L_{\max}}  \left ( g_{j(t)}(\mathcal{P}_S( x^s_{t})_{\mathcal{G}_{j(t)}}) - g_{\mathcal{G}_{j(t)}}( x^s_{t+1})_{\mathcal{G}_{j(t)}}) \right )
 \end{eqnarray}
 where the first inequality comes from the definition of function $\mathcal{P}_{S(x)} = \arg \min_{y\in S }\| y-x \|^2$,  and the second inequality uses (\ref{lemmaequ0.05}) in Lemma \ref{lemma0.6}. For the expectation of $T_1$, we have
 \begin{eqnarray} \label{equ_lema1_8}
 && \mathbb{E}(T_1) =\mathbb{E}  {  \left ( \langle \left (  \mathcal{P}_S( x^s_{t}) -x_{t}^s 　\right )_{\mathcal{G}_{j(t)}},( \widehat{v}_t^s)_{\mathcal{G}_{j(t)}} \rangle \right ) }
  \\ & = & \nonumber  \frac{1}{k}\mathbb{E}  \langle  \mathcal{P}_S( x^s_{t}) -x_{t}^s 　, \widehat{v}_t^s \rangle
   \\ & = & \nonumber \frac{1}{k}\mathbb{E}  \langle  \mathcal{P}_S( x^s_{t}) -x_{t}^s 　, \nabla f_{i_{t}}(\widehat{x}^{s}_{t}) - \nabla   f_{i_{t}}(\widetilde{x}^s) + \nabla f(\widetilde{x}^{s})  \rangle
  \\ & = & \nonumber \frac{1}{k}\mathbb{E}  \langle  \mathcal{P}_S( x^s_{t}) -x_{t}^s 　, \nabla f_{i_{t}}(\widehat{x}^{s}_{t}) \rangle +
    \frac{1}{k} \langle \mathbb{E} ( \mathcal{P}_S( x^s_{t}) -x_{t}^s )　, \mathbb{E} (- \nabla f_{i_{t}}(\widetilde{x}^s) + \nabla f(\widetilde{x}^{s}) ) \rangle
   \\ & = & \nonumber \frac{1}{k}\mathbb{E}  \langle  \mathcal{P}_S( x^s_{t}) -\widehat{x}_{t}^s 　, \nabla f_{i_{t}}(\widehat{x}^{s}_{t}) \rangle +  \frac{1}{k}\mathbb{E}  \langle  \widehat{x}_{t}^s -x_{t}^s 　, \nabla f_{i_{t}}(\widehat{x}^{s}_{t}) \rangle
   \\ & = & \nonumber \frac{1}{k}\mathbb{E}  \langle  \mathcal{P}_S( x^s_{t}) -\widehat{x}_{t}^s 　, \nabla f_{i_{t}}(\widehat{x}^{s}_{t}) \rangle +
  \frac{1}{k}\mathbb{E} \sum_{t'=0}^{ | K(t)|-1}  \langle \widehat{x}_{t,t'}^s -\widehat{x}_{t,t'+1}^s ,\nabla f_{i_{t}}(\widehat{x}^{s}_{t})  \rangle
     \\ & = & \nonumber \frac{1}{k}\mathbb{E}  \langle  \mathcal{P}_S( x^s_{t}) -\widehat{x}_{t}^s 　, \nabla f_{i_{t}}(\widehat{x}^{s}_{t}) \rangle +
  \frac{1}{k}\mathbb{E} \sum_{t'=0}^{ | K(t)|-1}  \langle \widehat{x}_{t,t'}^s -\widehat{x}_{t,t'+1}^s ,\nabla f_{i_{t}}(\widehat{x}^{s}_{t,t'})  \rangle
    \\ &  & \nonumber +  \frac{1}{k}\mathbb{E} \sum_{t'=0}^{ | K(t)|-1}  \langle \widehat{x}_{t,t'}^s -\widehat{x}_{t,t'+1}^s ,\nabla f_{i_{t}}(\widehat{x}^{s}_{t}) -\nabla f_{i_{t}}(\widehat{x}^{s}_{t,t'})  \rangle
  \\ & \leq  & \nonumber \frac{1}{k}\mathbb{E} \left ( f_{i_t}( \mathcal{P}_S( x_t^s ))- f_{i_{t}}(\widehat{x}^{s}_{t})  \right )  +  \frac{1}{k}\mathbb{E} \sum_{t'=0}^{ | K(t)|-1}  \langle \widehat{x}_{t,t'}^s -\widehat{x}_{t,t'+1}^s ,\nabla f_{i_{t}}(\widehat{x}^{s}_{t}) -\nabla f_{i_{t}}(\widehat{x}^{s}_{t,t'})  \rangle
    \\ &  & \nonumber +
  \frac{1}{k}\mathbb{E} \sum_{t'=0}^{ | K(t)|-1} \left (  f_{i_{t}}(\widehat{x}^{s}_{t,t'})  - f_{i_{t}}(\widehat{x}^{s}_{t,t'+1}) + \frac{L_{\max}}{2} \|  \widehat{x}_{t,t'}^s -\widehat{x}_{t,t'+1}^s  \|^2 \right )
\\ & =  & \nonumber \frac{1}{k}\mathbb{E} \left ( f_{i_t}( \mathcal{P}_S( x_t^s ))- f_{i_{t}}({x}^{s}_{t})  \right )  +  \frac{1}{k}\mathbb{E} \sum_{t'=0}^{ | K(t)|-1}  \langle \widehat{x}_{t,t'}^s -\widehat{x}_{t,t'+1}^s ,\nabla f_{i_{t}}(\widehat{x}^{s}_{t}) -\nabla f_{i_{t}}(\widehat{x}^{s}_{t,t'})  \rangle
\\ &  & \nonumber +
\frac{L_{\max}}{2k}\mathbb{E} \sum_{t'=0}^{ | K(t)|-1} \left (  \|  \widehat{x}_{t,t'}^s -\widehat{x}_{t,t'+1}^s  \|^2 \right )
\\ & =  & \nonumber \frac{1}{k}\mathbb{E} \left ( f_{i_t}( \mathcal{P}_S( x_t^s ))- f_{i_{t}}({x}^{s}_{t})  \right )  +  \frac{L_{\max}}{2k}\mathbb{E} \sum_{t'=0}^{ | K(t)|-1} \left (  \|  \widehat{x}_{t,t'}^s -\widehat{x}_{t,t'+1}^s  \|^2 \right )
\\ &  & \nonumber + \frac{1}{k}\mathbb{E} \sum_{t'=0}^{ | K(t)|-1}  \langle \widehat{x}_{t,t'}^s -\widehat{x}_{t,t'+1}^s ,\sum_{t''=0}^{t'-1} \nabla f_{i_{t}}(\widehat{x}^{s}_{t''}) -\nabla f_{i_{t}}(\widehat{x}^{s}_{t,t''+1})  \rangle
\\ & \leq  & \nonumber \frac{1}{k}\mathbb{E} \left ( f_{i_t}( \mathcal{P}_S( x_t^s ))- f_{i_{t}}({x}^{s}_{t})  \right )  +  \frac{L_{\max}}{2k}\mathbb{E} \sum_{t'=0}^{ | K(t)|-1} \left (  \|  \widehat{x}_{t,t'}^s -\widehat{x}_{t,t'+1}^s  \|^2 \right )
\\ &  & \nonumber + \frac{L_{\max}}{k}\mathbb{E} \sum_{t'=0}^{ | K(t)|-1} \left ( \left \| \widehat{x}_{t,t'}^s -\widehat{x}_{t,t'+1}^s \right \| \sum_{t''=0}^{t'-1} \left \| \widehat{x}^{s}_{t''} -\widehat{x}^{s}_{t,t''+1} \right \|  \right )
\\ & =  & \nonumber \frac{1}{k}\mathbb{E} \left ( f_{i_t}( \mathcal{P}_S( x_t^s ))- f_{i_{t}}({x}^{s}_{t})  \right )  +  \frac{L_{\max}}{2k}\mathbb{E} \sum_{t'=0}^{ | K(t)|-1} \left (  \|  \widehat{x}_{t,t'}^s -\widehat{x}_{t,t'+1}^s  \|^2 \right )
\\ &  & \nonumber + \frac{L_{\max}}{2k}\mathbb{E} \left ( \left ( \sum_{t'=0}^{ | K(t)|-1} \left  \| \widehat{x}_{t,t'}^s -\widehat{x}_{t,t'+1}^s \right \| \right )^2
-  \sum_{t'=0}^{ | K(t)|-1} \left (  \|  \widehat{x}_{t,t'}^s -\widehat{x}_{t,t'+1}^s  \|^2 \right ) \right )
\\ & =  & \nonumber \frac{1}{k}\mathbb{E} \left ( f_{i_t}( \mathcal{P}_S( x_t^s ))- f_{i_{t}}({x}^{s}_{t})  \right )  + \frac{L_{\max}}{2k}\mathbb{E} \left ( \sum_{t'=0}^{ | K(t)|-1} \left  \| \widehat{x}_{t,t'}^s -\widehat{x}_{t,t'+1}^s \right \| \right )^2
\\ & \leq & \nonumber \frac{1}{k}\mathbb{E} \left ( f_{i_t}( \mathcal{P}_S( x_t^s ))- f_{i_{t}}({x}^{s}_{t})  \right )  + \frac{L_{\max} \tau}{2k}\mathbb{E} \sum_{t'=0}^{ | K(t)|-1} \left  \| \widehat{x}_{t,t'}^s -\widehat{x}_{t,t'+1}^s \right \|^2
\\ & = & \nonumber \frac{1}{k}\mathbb{E} \left ( f_{i_t}( \mathcal{P}_S( x_t^s ))- f_{i_{t}}({x}^{s}_{t})  \right )  + \frac{L_{\max} \tau}{2k} \sum_{t'=0}^{ | K(t)|-1} \mathbb{E}  \left  \| B_{t'}^{s} \Delta^{s}_{t'}  \right \|^2
\\ & \leq & \nonumber \frac{1}{k}\mathbb{E} \left ( f_{i_t}( \mathcal{P}_S( x_t^s ))- f_{i_{t}}({x}^{s}_{t})  \right )  + \frac{L_{\max} \tau}{2k^2} \sum_{t'=0}^{ | K(t)|-1} \mathbb{E}  \left  \| \overline{x}^s_{t'+1} -x_{t'}^s    \right \|^2
\\ & \leq & \nonumber \frac{1}{k}\mathbb{E} \left ( f_{i_t}( \mathcal{P}_S( x_t^s ))- f_{i_{t}}({x}^{s}_{t})  \right )  + \frac{L_{\max} \tau}{2k^2} \sum_{t'=1}^{\tau }  \rho^{t'}  \mathbb{E}\left  \| \overline{x}^s_{t+1} -x_{t}^s    \right \|^2
\\ & \leq & \nonumber \frac{1}{k}\mathbb{E} \left  (f( \mathcal{P}_S( x_t^s ))- f(x_t^s) \right ) + \frac{L_{\max  }\tau \theta'}{2k^2} \mathbb{E}(\| x_t^s-\overline{x}_{t+1}^s \|^2)
  \end{eqnarray}
   where the fifth equality comes from that $x^s_t$ is independent to $i_t$, the sixth equality uses Lemma \ref{lemma0.5}, the first inequality uses the convexity of $f_i$ and (\ref{coordinate_lipschitz_constant2}), the second inequality uses (\ref{coordinate_lipschitz_constant1}). For the expectation of $T_2$, we have
 \begin{eqnarray} \label{equ_lema1_9}
 && \mathbb{E}(T_2) =\mathbb{E} \langle \left (  \Delta_{t}^s　\right )_{\mathcal{G}_{j(t)}},( \widehat{v}_t^s)_{\mathcal{G}_{j(t)}} \rangle
 \\ & = & \nonumber   \mathbb{E} \langle \left ( \Delta_{t}　\right )_{\mathcal{G}_{j(t)}}, \left ( \nabla f_{i_{t}}(\widehat{x}^{s}_{t}) - \nabla f_{i_{t}}(\widetilde{x}^s) + \nabla f(\widetilde{x}^{s}) \right )_{\mathcal{G}_{j(t)}} \rangle
 \\ & = & \nonumber   \mathbb{E} \langle \left ( \Delta_{t}　\right )_{\mathcal{G}_{j(t)}}, \left ( \nabla f_{i_{t}}(\widehat{x}^{s}_{t}) -\nabla f_{i_{t}}({x}^{s}_{t}) + \nabla f_{i_{t}}({x}^{s}_{t})  - \nabla f_{i_{t}}(\widetilde{x}^s) + \nabla f(\widetilde{x}^{s}) \right )_{\mathcal{G}_{j(t)}} \rangle
 \\ & = & \nonumber   \mathbb{E} \langle \left ( \Delta_{t}　\right )_{\mathcal{G}_{j(t)}},  \left (\nabla f_{i_{t}}(\widehat{x}^{s}_{t}) -\nabla f_{i_{t}}({x}^{s}_{t}) \right  )_{\mathcal{G}_{j(t)}} \rangle +
  \\ &  & \nonumber \mathbb{E} \langle \left ( \Delta_{t}　\right )_{\mathcal{G}_{j(t)}}, \left  (\nabla f_{i_{t}}({x}^{s}_{t}) - \nabla f_{i_{t}}(\widetilde{x}^s) \right )_{\mathcal{G}_{j(t)}}  \rangle +
 \mathbb{E} \langle \left ( \Delta_{t}　\right )_{\mathcal{G}_{j(t)}},   \nabla_{\mathcal{G}_{j(t)}} f (\widetilde{x}^s) \rangle
 \\ & \leq & \nonumber  \frac{1}{k} \mathbb{E} \left ( \| \overline{x}^s_{t+1} -x_t^s \|  \left \| \nabla f_{i_{t}}(\widehat{x}^{s}_{t}) -\nabla f_{i_{t}}({x}^{s}_{t}) \right \| \right ) +
  \\ &  & \nonumber\frac{1}{k} \mathbb{E} \left ( \| \overline{x}^s_{t+1} -x_t^s \|  \left \| \nabla f_{i_{t}}({x}^{s}_{t}) - \nabla f_{i_{t}}(\widetilde{x}^s) \right \| \right )
  + \mathbb{E} \langle \left ( \Delta_{t}　\right )_{\mathcal{G}_{j(t)}},   \nabla_{\mathcal{G}_{j(t)}} f (\widetilde{x}^s) \rangle
\\ & \leq & \nonumber   \frac{L_{res}}{k} \left ( \sum_{t' \in K(t)}  \| \overline{x}^s_{t+1} -x_t^s  \| \|  \Delta^{s}_{t'}   \|\right )
\\ \nonumber &  &  \frac{L_{nor }}{k} \mathbb{E} \left (  \| \overline{x}^s_{t+1} -x_t^s  \| \| {x}^s_{t} - {x}^s  \|  \right )
  + \mathbb{E} \langle \left ( \Delta_{t}　\right )_{\mathcal{G}_{j(t)}},   \nabla_{\mathcal{G}_{j(t)}} f (\widetilde{x}^s) \rangle
\\ & \leq & \nonumber   \frac{L_{res}}{k} \left ( \sum_{t' \in K(t)}  \| \overline{x}^s_{t+1} -x_t^s  \| \|  \Delta^{s}_{t'}   \|\right )
\\ \nonumber &  &  \frac{L_{nor }}{k} \mathbb{E} \left ( \sum_{t'=0}^{t-1} \| \overline{x}^s_{t+1} -x_t^s  \| \|  \Delta^{s}_{t'}   \|  \right )
 + \mathbb{E} \langle \left ( \Delta_{t}　\right )_{\mathcal{G}_{j(t)}},   \nabla_{\mathcal{G}_{j(t)}} f (\widetilde{x}^s) \rangle
\\ & \leq & \nonumber   \frac{L_{res}}{k^{3/2}} \sum_{t'=t-\tau}^{t-1} \rho^{(t-t')/2} \mathbb{E}(\| x_t^s-\overline{x}_{t+1}^s \|^2)
\\ &  & \nonumber  + \frac{L_{nor}}{k^{3/2}} \sum_{t'=0}^{t-1} \rho^{(t-t')/2} \mathbb{E}(\| x_t^s-\overline{x}_{t+1}^s \|^2)
  + \mathbb{E} \langle \left ( \Delta_{t}　\right )_{\mathcal{G}_{j(t)}},   \nabla_{\mathcal{G}_{j(t)}} f (\widetilde{x}^s) \rangle
\\ & = & \nonumber   \frac{1}{k^{3/2}} \left (L_{res} \theta_1+ L_{nor} \theta_2  \right ) \mathbb{E}(\| x_t^s-\overline{x}_{t+1}^s \|^2)
  + \mathbb{E} \langle \left ( \Delta_{t}　\right )_{\mathcal{G}_{j(t)}},   \nabla_{\mathcal{G}_{j(t)}} f (\widetilde{x}^s) \rangle
 \\ & \leq & \nonumber   \frac{L_{res} \theta_1+ L_{nor} \theta_2}{k^{3/2}}  \mathbb{E}(\| x_t^s-\overline{x}_{t+1}^s \|^2)
 + \mathbb{E} \langle \left ( \Delta_{t}　\right )_{\mathcal{G}_{j(t)}},   \nabla_{\mathcal{G}_{j(t)}} f (\widetilde{x}^s) \rangle
 % \\ & = & \nonumber  \mathbb{E}   \langle \left ( \Delta_{t}　\right )_{j(t)}　, \nabla_{j(t)} f_{i_{t}}({x}^{s}_{t})\rangle  +
 % \\ &  & \nonumber \mathbb{E} \langle \left ( \Delta_{t}　\right )_{j(t)}　, \nabla_{j(t)} f_{i_{t}}(\widehat{x}^{s}_{t}) - \nabla_{j(t)} f_{i_{t}}({x}^{s}_{t})  \rangle
 %\\ & \leq & \nonumber  \mathbb{E}  \left ( f_{i_t} (x_t^s) -f_{i_t} (x_{t+1}^s) + \frac{L_{\max}}{2} |(\Delta_{t})_{j(t)}|^2 \right )  +
 % \\ &  & \nonumber \mathbb{E} \langle \left ( \Delta_{t}　\right )_{j(t)}　, \nabla_{j(t)} f_{i_{t}}(\widehat{x}^{s}_{t}) - \nabla_{j(t)} f_{i_{t}}({x}^{s}_{t})  \rangle
 % \\ & \leq & \nonumber  \mathbb{E}  f (x_t^s) - \mathbb{E} f (x_{t+1}^s) + \frac{L_{\max}}{2n}  \mathbb{E}(\| x_t^s-\overline{x}_{t+1}^s \|^2) +
 % \\ &  & \nonumber \mathbb{E} \langle \left ( \Delta_{t}　\right )_{j(t)}　, \nabla_{j(t)} f_{i_{t}}(\widehat{x}^{s}_{t}) - \nabla_{j(t)} f_{i_{t}}({x}^{s}_{t})  \rangle
 %  \\ & \leq & \nonumber  \mathbb{E}  f (x_t^s) - \mathbb{E} f (x_{t+1}^s) + \frac{L_{\max}}{2n}  \mathbb{E}(\| x_t^s-\overline{x}_{t+1}^s \|^2) +
 % \\ &  & \nonumber \frac{L_{res}}{n^{3/2}} \sum_{t'=t-\tau}^{t-1} \rho^{(t-t')/2} \mathbb{E}(\| x_t^s-\overline{x}_{t+1}^s \|^2)
 %   \\ & = & \nonumber  \mathbb{E}  f (x_t^s) - \mathbb{E} f (x_{t+1}^s) + \frac{L_{\max}}{2n}  \mathbb{E}(\| x_t^s-\overline{x}_{t+1}^s \|^2) +
 % \\ &  & \nonumber \frac{L_{res}  \theta'' }{n^{3/2}} \mathbb{E}(\| x_t^s-\overline{x}_{t+1}^s \|^2)
\end{eqnarray}
where the second inequality uses Lemma \ref{lemma0.5}, the fourth inequality uses (\ref{equ_lema1_2.11}).
By substituting the upper bounds from (\ref{equ_lema1_8}) and (\ref{equ_lema1_9}) into (\ref{equ_lema1_7}), we have
\begin{eqnarray} \label{equ_lema1_10}
&& \mathbb{E}\| x_{t+1}^s - \mathcal{P}_S( x^s_{t+1}) \|^2
\\ & \leq & \nonumber \mathbb{E}  \| x_{t}^s - \mathcal{P}_S( x^s_{t}) \|^2  - \frac{1}{k}  \mathbb{E}(\| x_t^s-\overline{x}_{t+1}^s \|^2)+
  \\ &  & \nonumber \frac{2 \gamma}{L_{\max}k}  \mathbb{E} \left  (f( \mathcal{P}_S( x_t^s ))- f(x_t^s) \right ) + \frac{ \gamma\tau \theta'}{k^2} \mathbb{E}(\| x_t^s-\overline{x}_{t+1}^s \|^2)
 \\ &  & \nonumber +  \frac{2 \gamma}{L_{\max}} \left ( \frac{L_{res} \theta_1+ L_{nor} \theta_2}{k^{3/2}} \theta \mathbb{E}(\| x_t^s-\overline{x}_{t+1}^s \|^2) + \mathbb{E} \langle \left ( \Delta_{t}　\right )_{\mathcal{G}_{j(t)}},   \nabla_{\mathcal{G}_{j(t)}} f (\widetilde{x}^s) \rangle  \right .
  \\ &  & \nonumber \left .  \frac{1}{k} \mathbb{E} g(\mathcal{P}_S( x^s_{t})) -  \mathbb{E} g( x^s_{t+1})  + \frac{k-1}{k} \mathbb{E} g( x^s_{t})
  \right )
    \\ & = & \nonumber   \mathbb{E}  \| x_{t}^s - \mathcal{P}_S( x^s_{t}) \|^2 +\frac{2 \gamma}{L_{\max}k}  \mathbb{E} \left  (f( \mathcal{P}_S( x_t^s ))- f(x_t^s) \right )-
 \\ &  & \nonumber   \frac{1}{k} \left (  1- \frac{ \gamma\tau \theta'}{k} - \frac{2 ({L_{res} \theta_1+ L_{nor} \theta_2}) \gamma}{k^{1/2} L_{\max}}\right ) \mathbb{E}(\| x_t^s-\overline{x}_{t+1}^s \|^2)
 + \frac{2 \gamma}{L_{\max}} \left (  \mathbb{E} \langle \left ( \Delta_{t}　\right )_{\mathcal{G}_{j(t)}},   \nabla_{\mathcal{G}_{j(t)}} f (\widetilde{x}^s) \rangle  \right .
  \\ &  & \nonumber \left .  \frac{1}{k} \mathbb{E} g(\mathcal{P}_S( x^s_{t})) -  \mathbb{E} g( x^s_{t+1})  + \frac{k-1}{k} \mathbb{E} g( x^s_{t})
  \right )
\end{eqnarray}
We consider a fixed stage $s+1$ such that $x_0^{s+1} = x_{m}^{s}$. By summing
the the inequality (\ref{equ_lema1_10}) over $t = 0,\cdots,m-1$,  we obtain
\begin{eqnarray}
\label{equ_lema1_10}
&& \mathbb{E}\| x^{s+1} - \mathcal{P}_S( x^{s+1}) \|^2
\\ & \leq & \nonumber  \mathbb{E}\| x^{s} - \mathcal{P}_S( x^{s}) \|^2 +
\sum_{t'=0}^{m-1}\frac{2 \gamma}{L_{\max}k}  \mathbb{E} \left  (f( \mathcal{P}_S( x_{t'}^{s+1} ))- f(x_{t'}^{s+1}) \right ) -
\\ &  & \nonumber  \sum_{t'=0}^{m-1} \frac{1}{k} \left (  1- \frac{ \gamma\tau \theta'}{k} - \frac{2 ({L_{res} \theta_1+ L_{nor} \theta_2})\gamma}{k^{1/2} L_{\max}}\right ) \cdot
\mathbb{E}(\| x_{t'}^{s+1}-\overline{x}_{t'+1}^{s+1} \|^2)
\\ &  & \nonumber + \frac{2 \gamma}{L_{\max}} \sum_{t'=0}^{m-1}   \mathbb{E} \left \langle \left ( \Delta_{t'}　\right )_{\mathcal{G}_{j(t')}},   \nabla_{\mathcal{G}_{j(t')}} f (\widetilde{x}^s) \right \rangle
\\ &  & \nonumber +  \frac{2 \gamma}{L_{\max}}  \sum_{t'=0}^{m-1} \left ( \frac{1}{k} \mathbb{E} g(\mathcal{P}_S( x^{s+1}_{t'})) -  \mathbb{E} g( x^{s+1}_{t'+1})   + \frac{n-1}{k} \mathbb{E} g( x^{s+1}_{t'})
  \right )
\\ & = & \nonumber  \mathbb{E}\| x^{s} - \mathcal{P}_S( x^{s}) \|^2 +
 \sum_{t'=0}^{m-1}\frac{2 \gamma}{L_{\max}k}  \mathbb{E} \left  (f( \mathcal{P}_S( x_{t'}^{s+1} ))- f(x_{t'}^{s+1}) \right ) -
\\ &  & \nonumber  \sum_{t'=0}^{m-1} \frac{1}{k} \left (  1- \frac{ \gamma\tau \theta'}{k} - \frac{2 ({L_{res} \theta_1+ L_{nor} \theta_2}) \gamma}{k^{1/2} L_{\max}}\right ) \cdot
 \mathbb{E}(\| x_{t'}^{s+1}-\overline{x}_{t'+1}^{s+1} \|^2)
\\ &  & \nonumber + \frac{2 \gamma}{L_{\max}}   \mathbb{E} \left \langle x^{s+1}-x^{s},   \nabla f (\widetilde{x}^s) \right \rangle
\\ &  & \nonumber +  \frac{2 \gamma}{L_{\max}}  \sum_{t'=0}^{m-1} \left ( \frac{1}{k} \mathbb{E} g(\mathcal{P}_S( x^{s+1}_{t'})) -  \mathbb{E} g( x^{s+1}_{t'+1})    + \frac{k-1}{k} \mathbb{E} g( x^{s+1}_{t'})
  \right )
\\ & \leq & \nonumber  \mathbb{E}\| x^{s} - \mathcal{P}_S( x^{s}) \|^2 +
\sum_{t'=0}^{m-1}\frac{2 \gamma}{L_{\max}k}  \mathbb{E} \left  (f( \mathcal{P}_S( x_{t'}^{s+1} ))- f(x_{t'}^{s+1}) \right ) -
\\ &  & \nonumber  \sum_{t'=0}^{m-1} \frac{1}{k} \left (  1- \frac{ \gamma\tau \theta'}{k} - \frac{2 ({L_{res} \theta_1+ L_{nor} \theta_2}) \gamma}{k^{1/2} L_{\max}}\right )  \cdot
 \mathbb{E}(\| x_{t'}^{s+1}-\overline{x}_{t'+1}^{s+1} \|^2)
\\ &  & \nonumber + \frac{2 \gamma}{L_{\max}}   \mathbb{E} \left (  f ({x}^s) - f ({x}^{s+1}) + \frac{L_{nor}}{2} \| x^{s+1}-{x}^{s} \|^2 \right )
\\ &  & \nonumber +  \frac{2 \gamma}{L_{\max}}  \sum_{t'=0}^{m-1} \left ( \frac{1}{k} \mathbb{E} g(\mathcal{P}_S( x^{s+1}_{t'})) -  \mathbb{E} g( x^{s+1}_{t'+1})  + \frac{k-1}{k} \mathbb{E} g( x^{s+1}_{t'})
  \right )
\\ & = & \nonumber  \mathbb{E}\| x^{s} - \mathcal{P}_S( x^{s}) \|^2 +
 \sum_{t'=0}^{m-1}\frac{2 \gamma}{L_{\max}k}  \mathbb{E} \left  (f( \mathcal{P}_S( x_{t'}^{s+1} ))- f(x_{t'}^{s+1}) \right ) -
\\ &  & \nonumber  \sum_{t'=0}^{m-1} \frac{1}{k} \left (  1- \frac{ \gamma\tau \theta'}{k} - \frac{2 ({L_{res} \theta_1+ L_{nor} \theta_2}) \gamma}{k^{1/2} L_{\max}}\right ) \cdot
 \mathbb{E}(\| x_{t'}^{s+1}-\overline{x}_{t'+1}^{s+1} \|^2)
\\ &  & \nonumber + \frac{2 \gamma}{L_{\max}}  \sum_{t'=0}^{m-1} \mathbb{E} \left (  f ({x}_{t'}^{s+1}) - f ({x}_{t+1}^{s+1}) \right )
 + \frac{ L_{nor} \gamma}{L_{\max}} \mathbb{E}   \left  \| \sum_{t'=0}^{m-1}\left ( x_{t'}^{s+1}-{x}_{t'+1}^{s+1} \right ) \right  \|^2
\\ &  & \nonumber +  \frac{2 \gamma}{L_{\max}}  \sum_{t'=0}^{m-1} \left ( \frac{1}{k} \mathbb{E} g(\mathcal{P}_S( x^{s+1}_{t'})) -  \mathbb{E} g( x^{s+1}_{t'+1})   + \frac{k-1}{k} \mathbb{E} g( x^{s+1}_{t'})
  \right )
\\ & \leq & \nonumber  \mathbb{E}\| x^{s} - \mathcal{P}_S( x^{s}) \|^2 + \frac{2 \gamma}{L_{\max}k}\sum_{t'=0}^{m-1}  \left ( F^* - \mathbb{E}F(x_{t'}^{s+1}) \right )
  + \frac{2 \gamma}{L_{\max}} \sum_{t'=0}^{m-1} \left ( \mathbb{E}F(x_{t'}^{s+1}) - \mathbb{E}F(x_{t'+1}^{s+1}) \right )
\\ &  & \nonumber -  \sum_{t'=0}^{m-1} \frac{1}{k} \left (  1-  \Lambda_{nor} \gamma - \frac{ \gamma\tau \theta'}{n} - \frac{2 (\Lambda_{res} \theta_1 +\Lambda_{nor} \theta_2 ) \gamma}{k^{1/2} }\right ) \cdot
 \mathbb{E}(\| x_{t'}^{s+1}-\overline{x}_{t'+1}^{s+1} \|^2)
\\ & \leq & \nonumber  \mathbb{E}\| x^{s} - \mathcal{P}_S( x^{s}) \|^2 + \frac{2 \gamma}{L_{\max}k}\sum_{t'=0}^{m-1}  \left ( F^* - \mathbb{E}F(x_{t'}^{s+1}) \right )
   + \frac{2 \gamma}{L_{\max}}  \left ( \mathbb{E}F(x^{s}) - \mathbb{E}F(x^{s+1}) \right )
\end{eqnarray}
where the second inequality uses (\ref{equdef1}),  the final inequality comes from $ 1-  \Lambda_{nor} \gamma - \frac{ \gamma\tau \theta'}{n} - \frac{2 (\Lambda_{res} \theta_1 +\Lambda_{nor} \theta_2 ) \gamma}{n^{1/2} } \geq 0$. Define $S(x^s) =  \mathbb{E}  \| x_{t} - \mathcal{P}_S( x^s) \|^2  +  \frac{2 \gamma}{L_{\max}} \mathbb{E} \left ( F( x^s) - F^* \right )  $. According to (\ref{equ_lema1_10}), we have
\begin{eqnarray} \label{equ_lema1_12}
 S(x^{s+1})
  \leq   S(x^s)  - \frac{2 \gamma}{L_{\max}k}\sum_{t'=0}^{m-1} \mathbb{E} \left ( F( x^{s+1}_{t'}) - F^* \right )
 \leq    S(x^s)  - \frac{2 m \gamma}{ L_{\max} k}\mathbb{E} \left ( F( x^{s+1}) - F^* \right )
\end{eqnarray}
where  the second inequality comes from the monotonicity of $\mathbb{E}  F(x_{t}^s)$. According to (\ref{equ_lema1_12}), we have
\begin{eqnarray} \label{equ_lema1_13}
S(x^{s})
 \leq    S(x^0)  - \frac{2 m \gamma s}{ L_{\max} k}\mathbb{E} \left ( F( x^{s}) - F^* \right )
\end{eqnarray}
Thus, the sublinear convergence rate (\ref{theorem_equ_1}) for general smooth convex function $f$  can be obtained from (\ref{equ_lema1_13}).

If the optimal strong convexity for the smooth convex function $f$   holds with $l>0$, we have (\ref{equ_lema1_13.1}) as proved in (A.28) of \cite{liu2015asynchronous}.
\begin{eqnarray} \label{equ_lema1_13.1}
\mathbb{E} \left ( F( x^{s}) - F^* \right ) \geq \frac{L_{\max } l}{2(l \gamma +L_{\max})} S(x^{s})
\end{eqnarray}
Thus, substituting (\ref{equ_lema1_14}) into (\ref{equ_lema1_12}), we have
\begin{eqnarray} \label{equ_lema1_14}
S(x^{s+1})
 \leq   S(x^s)  - \frac{2 m \gamma l}{2k(l \gamma +L_{\max})} S(x^{s+1})
\end{eqnarray}
Based on (\ref{equ_lema1_14}), we have (\ref{equ_lema1_15}) by induction.
\begin{eqnarray} \label{equ_lema1_15}
S(x^{s})
 \leq  \left ( \frac{1}{1+\frac{2 m \gamma l}{2k(l \gamma +L_{\max})}} \right )^s S(x^0)
\end{eqnarray}
Thus, the linear convergence rate (\ref{theorem_equ_2}) for  the optimal strong convexity on $f$ can be obtained from (\ref{equ_lema1_15}).
\end{proof}

\section{Conclusion}\label{conclusion}
In this paper, we propose an asynchronous stochastic  block coordinate descent algorithm with the accelerated technology of variance reduction  (AsySBCDVR), which are with lock-free in the implementation and analysis. AsySBCDVR is particularly  important because it  can scale well with the sample size and dimension simultaneously.  We prove that AsySBCDVR achieves a linear convergence rate when the function $f$ is with the optimal strong
convexity property, and a sublinear rate  when $f$ is with the  general convexity. More importantly, a near-linear speedup on a parallel system with shared memory can be obtained.

\nocite{langley00}

%\section*{Acknowledgments}
%This work was supported
%by the Natural Sciences and Engineering Research Council of
%Canada (NSERC) and the NSF of China (61232016, 61573191 and 61202137).

\bibliography{sample}

\end{document}